\pdfoutput=1
\documentclass[11pt, letterpaper]{article}
\usepackage{fullpage}
\usepackage{blindtext}
\usepackage{hyperref}
\hypersetup{
	colorlinks=true,
	linkcolor=blue!70!black,
	citecolor=blue!70!black,
	urlcolor=blue!70!black
}

\usepackage[english]{babel}
\usepackage[T1]{fontenc}
\usepackage{tablefootnote}
\usepackage[table]{xcolor}
\usepackage{tabularx}
\newcolumntype{Y}{>{\centering\arraybackslash}X}
\usepackage{amsmath}
\usepackage{mathtools}
\usepackage{todonotes}
\usepackage{amssymb}
\usepackage{amsthm}
\usepackage{booktabs}
\usepackage{float}
\usepackage{accents}
 
\usepackage{bbm}
\usepackage{bm}
\usepackage{graphicx}
\graphicspath{{./figs/}}

\usepackage{cleveref}
\usepackage{thm-restate}
\usepackage[lined,boxed,ruled,norelsize,linesnumbered]{algorithm2e}

\newtheorem{theorem}{Theorem}
\newtheorem{lemma}{Lemma}
\newtheorem{fact}{Fact}
\newtheorem{definition}{Definition}
\newtheorem{corollary}{Corollary}
\newtheorem{proposition}{Proposition}

\newtheorem{assumption}{Assumption}
\newtheorem{remark}{Remark}
\newtheorem{model}{Model}

\renewcommand{\epsilon}{\varepsilon}
\newcommand{\defeq}{\coloneqq}
\newcommand{\norm}[1]{\left\lVert#1\right\rVert}
\newcommand{\norms}[1]{\lVert#1\rVert}
\newcommand{\normop}[1]{\left\lVert#1\right\rVert_{\textup{op}}}

\newcommand{\normf}[1]{\left\lVert#1\right\rVert_{\textup{F}}}

\newcommand{\normsop}[1]{\lVert#1\rVert_{\textup{op}}}
\newcommand{\normsf}[1]{\lVert#1\rVert_{\textup{F}}}

\newcommand{\inprod}[2]{\left\langle#1, #2\right\rangle}

\newcommand{\eps}{\epsilon}
\newcommand{\lam}{\lambda}
\newcommand{\0}{\boldsymbol{0}} 
 
\newcommand{\R}{\mathbb{R}}

\newcommand{\N}{\mathbb{N}}

\newcommand{\bas}[1]{\begin{align*}#1\end{align*}}

\newcommand{\half}{\frac{1}{2}}

\newcommand{\1}{\boldsymbol{1}}
\newcommand{\E}{\mathbb{E}}

\newcommand{\Nor}{\mathcal{N}}

\newcommand{\Tr}{\textup{Tr}}

\newcommand{\simiid}{\stackrel{\mathsf{iid}}{\sim}}

\newcommand{\ma}{\mathbf{A}}
\newcommand{\hmsig}{\widehat{\msig}}

\newcommand{\me}{\mathbf{E}}
\newcommand{\mh}{\mathbf{H}}

\newcommand{\id}{\mathbf{I}}
\newcommand{\bb}[1]{\left(#1\right)}

\newcommand{\dd}{\textup{d}}

\definecolor{burntorange}{rgb}{0.8, 0.33, 0.0}

\newcommand{\tO}{\widetilde{O}}

\newcommand{\rank}{\textup{rank}}

\newcommand{\Par}[1]{\left(#1\right)}
\newcommand{\Brack}[1]{\left[#1\right]}
\newcommand{\Brace}[1]{\left\{#1\right\}}

\newcommand{\alg}{\mathcal{A}}

\newcommand{\mb}{\mathbf{B}}
\newcommand{\mg}{\mathbf{G}}

\newcommand{\my}{\mathbf{Y}}
\newcommand{\mv}{\mathbf{V}}

\newcommand{\msig}{\boldsymbol{\Sigma}}

\newcommand{\mlam}{\boldsymbol{\Lambda}}

\newcommand{\ms}{\mathbf{S}}
\newcommand{\mq}{\mathbf{Q}}

\newcommand{\mx}{\mathbf{X}}

\newcommand{\mmu}{\mathbf{U}}

\newcommand{\mzero}{\mathbf{0}}

\newcommand{\Sym}{\mathbb{S}}

\newcommand{\PSD}{\Sym_{\succeq \mzero}}

\newcommand{\gam}{\gamma}

\newcommand{\codeStyle}[1]{{\bfseries #1} }
\newcommand{\codeInput}{\codeStyle{Input:}}	
	
\newcommand{\codeReturn}{\codeStyle{Return:}}	
	
\newcommand{\md}{\mathbf{D}}

\newcommand{\inner}{\inprod}

\newcommand{\poly}{\textup{poly}}

\newcommand{\ind}{\mathbb{I}}

\newcommand{\polylog}{\textup{polylog}}
\newcommand{\Oja}{\mathsf{Oja}}

\newcommand{\calE}{\mathcal{E}}

\newcommand{\TV}{\mathrm{TV}}

\newcommand{\mc}{\mathbf{C}}

\newcommand{\mi}{\mathbf{I}}
\newcommand{\mk}{\mathbf{K}}

\newcommand{\mpp}{\mathbf{P}}

\newcommand{\muu}{\mathbf{U}}
\newcommand{\mz}{\mathbf{Z}}

\newcommand{\vc}{\mathbf{c}}

\newcommand{\vg}{\mathbf{g}}

\newcommand{\vm}{\mathbf{m}}

\newcommand{\vs}{\mathbf{s}}
\newcommand{\vu}{\mathbf{u}}
\newcommand{\vv}{\mathbf{v}}
\newcommand{\vw}{\mathbf{w}}
\newcommand{\vx}{\mathbf{x}}
\newcommand{\vy}{\mathbf{y}}
\newcommand{\vz}{\mathbf{z}}

\newcommand{\calA}{\mathcal{A}}

\newcommand{\calD}{\mathcal{D}}

\newcommand{\calG}{\mathcal{G}}

\newcommand{\calN}{\mathcal{N}}

\newcommand{\op}{\textup{op}}

\newcommand{\Unif}{\mathsf{Unif}}
\newcommand{\KL}{\mathrm{KL}}

\newcommand{\vperp}{\mv_{\perp}}

\newcommand{\AG}{\mathsf{AnalyzeGauss}}

\newcommand{\msign}{\textup{m}_{\textup{sign}}}

\title{Gap-Free Streaming PCA Beyond Rank-One Updates: \\ Near-Optimal Rates and Applications to Differential Privacy}
\date{}
\author{
Anming Gu\thanks{University of Texas at Austin, \texttt{anminggu@cs.utexas.edu}}
\and
Syamantak Kumar\thanks{University of Texas at Austin, \texttt{syamantak@utexas.edu}}
\and
Kevin Tian\thanks{University of Texas at Austin, \texttt{kjtian@cs.utexas.edu}}
\and
Chutong Yang\thanks{University of Texas at Austin, \texttt{cyang98@utexas.edu}}
}
\begin{document}

\maketitle

\begin{abstract}
Streaming principal component analysis (PCA) seeks to recover a leading spectral subspace in a single pass over a data stream. We give a new analysis of the ubiquitous \emph{Oja's algorithm} \cite{Oja82} for the most general, \emph{gap-free} variant of this problem, where no eigengap assumptions are made on the underlying mean matrix, complemented by a nearly-matching lower bound. Prior works achieving near-optimal rates for streaming PCA either required gap assumptions \cite{jain2016streaming, HuangNW21}, or were limited to rank-one updates \cite{allen2017first, Liang23}. Our proof only uses a second moment bound on the individual stochastic updates, bypassing the almost sure bounds needed by prior near-optimal analyses, and the analogous offline matrix Bernstein bound. We also extend our result to a Rayleigh quotient notion of approximate PCA, addressing an open question of \cite{jain2016streaming}. As our main application, we give gap-free differentially private PCA guarantees for sub-Gaussian data, settling Conjecture 1.1 of \cite{Brown26} up to logarithmic factors.
\end{abstract}

\thispagestyle{empty}
\newpage

\begingroup
\setlength{\parskip}{0pt}
\tableofcontents
\endgroup

\thispagestyle{empty}
\newpage

\pagenumbering{arabic}
\section{Introduction}

Let $\ma_1,\ma_2,\ldots\in\R^{d\times d}$ be i.i.d.\ stochastic matrices with common mean $\msig\succeq\mzero_{d \times d}$. In \emph{streaming principal component analysis} (PCA), the goal is to recover a unit vector close to the largest eigenvector of $\msig$, while processing each update only once, ideally with small space overhead. 

The classical Oja's algorithm~\cite{Oja82} is perhaps the simplest method for this task: starting from a (randomly initialized) unit vector $\vw_0$, it repeatedly iterates
\[
    \vw_t
    \gets
    \frac{(\id_d+\eta_t\ma_t)\vw_{t-1}}
         {\norm{(\id_d+\eta_t\ma_t)\vw_{t-1}}_2}.
\]
This update can be performed using $O(d)$ auxiliary space, i.e., without storing a matrix explicitly. Oja's algorithm is extremely well-studied \cite{BalsubramaniDF13, Shamir16, jain2016streaming, allen2017first, HuangNW21, Liang23}, and is known to achieve near-optimal rates of convergence in various settings, under standard regularity assumptions on the sequence $\{\ma_t\}_{t \ge 1}$, such as a second moment bound and almost sure bound (Model~\ref{model:oja_prob1}). Notably, these are the same assumptions required by the \emph{matrix Bernstein} concentration inequality (cf.\ Proposition~\ref{prop:gapfree_offline}), which solves the same stochastic eigenvector estimation problem near-optimally, albeit in an offline setting and using $O(d^2)$ space.

We study the most general formulation of streaming PCA, where no gap assumptions are placed on $\msig$'s spectrum. 
PCA objectives become ill-conditioned when the leading eigenvalues are equal or close to equal, e.g., if $\lam_1(\msig) = \lam_2(\msig)$, then recovering the leading eigenvector is not even well-defined. A common alternative in such \emph{gap-free} settings, popularized by \cite{GarberH15, ZhuL16, allen2017first}, is to ask for a unit vector with little mass on eigenvectors whose eigenvalues are below $(1-\gam)\lambda_1(\msig)$, for a parameter $\gamma \in (0, 1)$. We formalize this correlation PCA (cPCA) objective in Definition~\ref{def:cpca}.

Perhaps surprisingly, all prior near-optimal rates for streaming PCA, via Oja's algorithm or otherwise, either required an eigengap assumption on $\msig$ \cite{jain2016streaming, HuangNW21}, or were limited to the setting where every $\ma_t$ is rank-one \cite{allen2017first, Liang23}. This motivates our work's central question.
\begin{gather*}\emph{Does Oja's algorithm achieve near-optimal convergence for streaming PCA,} \\
\emph{without eigengap assumptions on } \msig \emph{ or rank restrictions on the } \{\ma_t\}_{t \ge 1}?\end{gather*}

\subsection{Our results}
Our main result (Theorem~\ref{thm:gapfree_oja}) answers this question affirmatively. In fact, its convergence guarantee holds under qualitatively weaker regularity assumptions (Model~\ref{model:oja_special}) than used by prior work. Assuming a bound on the standard matrix variance parameter,
\[V \defeq \max\left\{
        \normop{\E[(\ma_t-\msig)(\ma_t-\msig)^\top]},
        \normop{\E[(\ma_t-\msig)^\top(\ma_t-\msig)]}
    \right\},\]
Theorem~\ref{thm:gapfree_oja} shows that with constant probability, Oja's algorithm returns a $(\gamma, \Delta)$-cPCA (i.e., has squared correlation at most $\Delta$ with the eigenspace below $(1 - \gamma)\lam_1(\msig)$), using\footnote{We use the notation $\tO$ to suppress polylogarithmic factors in problem parameters in informal rate summaries. All of our formal theorem statements specify all parameter dependences explicitly.}
\[\tO\Par{\frac{V}{\lam_1^2 \gamma^2\Delta} + \frac 1 \gamma}\]
online samples $\ma_t$. The first term in the above rate is complemented with a nearly-matching lower bound in Theorem~\ref{thm:pca_lower}, and the second term is a consequence of the standard convergence rate of the (offline) power method, in the special deterministic setting where all $\ma_t = \msig$. 

Interestingly, Theorem~\ref{thm:gapfree_oja} holds under weaker requirements than earlier convergence analyses of Oja's algorithm. In particular, it only posits a matrix variance bound $V$ (Model~\ref{model:oja_special}), and circumvents the almost sure bound (Model~\ref{model:oja_prob1}) typically used by prior works on streaming PCA, as well as the matrix Bernstein inequality. As a tradeoff, it only offers a constant success probability (more generally, Theorem~\ref{thm:gapfree_oja}'s sample complexity scales inverse-polynomially in the failure probability $\zeta$). In Theorem~\ref{thm:high_probability_oja_absolute}, we give an alternative result that leverages geometric aggregation to achieve a $\polylog(\frac 1 \zeta)$ sample complexity overhead. This result analyzes an extension of Oja's algorithm to block matrices (Algorithm~\ref{alg:boosted_sketch_oja}), and requires $d \cdot \polylog(\frac 1 \zeta)$ auxiliary space.

In Section~\ref{sec:epca}, we also consider the energy objective $\vw^\top\msig\vw\ge(1-\alpha)\lam_1$ (ePCA, Definition~\ref{def:epca}). While a black-box cPCA-to-ePCA conversion (Lemma 8, \cite{JambulapatiKLPPT24}) exists, its combination with Theorem~\ref{thm:gapfree_oja} leads to a suboptimal sample complexity by a factor of $\frac1\alpha$. Instead, we give a multiscale reduction-based analysis in Theorem~\ref{thm:gapfree_epca} that shows Oja's algorithm returns an $\alpha$-ePCA using \[
\widetilde{O}\Par{\frac{V}{\lam_1^2\alpha^2}+\frac1\alpha}
\]
samples. Here also, the first term is complemented with a nearly matching lower bound (Corollary~\ref{cor:epca_lower_bound}). This result addresses an open question posed by Section 6 of \cite{jain2016streaming}.

Finally, as our main application, we consider the setting of \emph{differentially private} PCA, i.e., where the goal is to solve PCA subject to $(\eps, \delta)$-DP (Definition~\ref{def:dp}). A prior work by \cite{LiuKJO22} achieved a near-optimal convergence rate for this problem under an eigengap. We give an analogous sample bound in the gap-free setting: for publicly known $\nu,\lam_1$, Theorem~\ref{thm:dp_pca_utility} returns an $(\eps,\delta)$-DP estimator that is a $(\gamma,\Delta)$-cPCA with high probability, using
\[\tO\Par{\frac{d\nu^4}{\gamma^2\lam_1^2\Delta}
+\frac{d\nu^2}{\eps\gamma\lam_1\sqrt\Delta}}\]
samples. Our result is stated directly under sub-Gaussianity (Definition~\ref{definition:hypercontractivity}). For Gaussian data, $\nu^2=\lam_1$, and the polynomial dependence matches Corollary 5.2 of \cite{LiuKJO22}, with spectral resolution $\gamma$ replacing the relative eigengap. We obtain a slightly better $\gamma$ dependence than \cite{LiuKJO22} by avoiding minibatches, instead taking full passes to obtain an improved sensitivity tradeoff. Its analysis uses R\'enyi differential privacy to compose the Gaussian queries and control adaptive clipping. Further, applying the ePCA analysis to the same algorithm gives, under the same $\nu$-sub-Gaussian model, an $\alpha$-ePCA with sample complexity (Theorem~\ref{thm:private_epca}),  
\[
\tO\Par{
\frac{d\nu^4}{\lam_1^2\alpha^2}
+
\frac{d\nu^2}{\eps\lam_1\alpha}
}.
\]
In particular, for Gaussian data with a publicly known $\lam_1$, a setting where $\nu^2=O(\lam_1)$, our new sample complexity bound above matches the rate conjectured by \cite{Brown26} up to logarithmic factors. The main outstanding questions left by Theorem~\ref{thm:private_epca} are to remove the remaining polylogarithmic overhead, and to privately estimate $\lam_1$ from samples.

\begin{table}[!htb]
    \centering
    \renewcommand{\arraystretch}{1.4}
    \setlength{\tabcolsep}{8pt}

    \begin{tabular*}{0.96\textwidth}{
        @{\extracolsep{\fill}}
        c c c c
        @{}
    }
        \toprule
        Work
        & General updates
        & Gap-free
        & Near-optimal rate \\
        \midrule

        \cite{Shamir16}, Corollary~1
        & $\checkmark$
        & $\checkmark$
        & -- \\[2pt]

        \cite{jain2016streaming}, Theorem~3
        & $\checkmark$
        & --
        & $\checkmark$ \\[2pt]

        \cite{allen2017first}, Theorem~2
        & --
        & $\checkmark$
        & $\checkmark$ \\[2pt]

        \cite{HuangNW21}, Theorem~3
        & $\checkmark$
        & --
        & $\checkmark$ \\[2pt]

        \cite{Liang23}, Theorem~3.3
        & --
        & $\checkmark$
        & $\checkmark$ \\[2pt]

        \textbf{This work}, Theorem~\ref{thm:gapfree_oja}
        & $\checkmark$
        & $\checkmark$
        & $\checkmark$ \\

        \bottomrule
    \end{tabular*}

    \caption{
        Representative streaming PCA guarantees.
        ``Near-optimal'' means the rate matches the lower bound (Theorem~\ref{thm:pca_lower}) up to logarithmic factors and low-order terms. ``General updates'' means no rank restrictions are placed, and only a statistical assumption (e.g., Models~\ref{model:oja_special} or~\ref{model:oja_prob1}) is used.
    }
    \label{tab:streaming_pca_comparison}
\end{table}

\subsection{Our techniques}
\label{ssec:intro_techniques}

Our main result, Theorem~\ref{thm:gapfree_oja}, follows from a new analysis of Oja's algorithm that leads to arguably a simpler convergence proof than in prior works, e.g., \cite{jain2016streaming}. We begin by overviewing this new strategy, and provide an overview of our auxiliary results (Theorems~\ref{thm:pca_lower},~\ref{thm:high_probability_oja_absolute},~\ref{thm:gapfree_epca},~\ref{thm:dp_pca_utility}, and~\ref{thm:private_epca}).

\textbf{Expected trace as a potential.}
Our analysis starts from the operator viewpoint of \cite{jain2016streaming}.
Writing the unnormalized Oja iterate as being induced by the random operator
\[
    \mb_t
    =
    (\id_d+\eta_t\ma_t)\cdots(\id_d+\eta_1\ma_1),
\]
their analysis controls the ratio between the energy of $\mb_t$ in the
orthogonal complement of the leading eigenvector $\vv_1$ and the energy
along $\vv_1$ (reproduced as Lemma~\ref{lem:gapfree_one_step_power}). 
At the population level, these two quantities evolve at rates governed by
$\lambda_2$ and $\lambda_1$, respectively, so their separation is driven
by the eigengap $\lambda_1-\lambda_2$.  This is precisely what becomes
problematic for a gap-free objective.

Our proof departs from this strategy, and 
instead compares $\mb_t$ with its population
counterpart
\[
    \mc_t
    =
    (\id_d+\eta_t\msig)\cdots(\id_d+\eta_1\msig)
    =\E[\mb_t].
\]
Let $\mpp$ denote the orthogonal projector onto eigenvectors with eigenvalues below
$(1-\gam)\lambda_1$. Our starting point is the following consequence of the triangle inequality,
\begin{equation}\label{eq:triangle_ineq_intro}
    \frac{\normf{\mpp\mb_t}}{\normf{\mb_t}}
    \leq
    \frac{\normf{\mpp\mc_t}}{\normf{\mc_t}}
    +
    \normf{
        \frac{\mb_t}{\normf{\mb_t}}
        -
        \frac{\mc_t}{\normf{\mc_t}}
    } \leq \frac{\normf{\mpp\mc_t}}{\normf{\mc_t}} + \frac{2\normf{\mb_t - \mc_t}}{\normf{\mc_t}},
\end{equation}
where the last inequality holds by a derivation in \eqref{eq:crazy_triangle_ineq}. The left-hand side above is precisely the quantity that Lemma~\ref{lem:gapfree_one_step_power} seeks to control in order to yield cPCA guarantees.

This inequality splits our bound into two terms: a deterministic center
(depending only on $\mc_t$), and the relative deviation of a random $\mb_t$. The first term is simple to control using analyses of the standard power method. To bound the second term, since $\E[\mb_t]=\mc_t$, we have $\E\normf{\mb_t-\mc_t}^2
    =
    \E[\Tr(\mb_t\mb_t^\top)]-\Tr(\mc_t\mc_t^\top)$, suggesting the use of $\E[\Tr(\mb_t\mb_t^\top)]$ as our potential.

The heart of our new analysis is Lemma~\ref{lem:expected_total_energy_vary_step}, which precisely achieves the required control of the expected trace, assuming only a matrix variance bound. Concretely, we show that under Model~\ref{model:oja_special}, Oja's algorithm with step sizes $\{\eta_s\}_{s \ge 1}$ satisfies
\[
    \E\Brack{\Tr(\mb_t\mb_t^\top)}
    \leq
    \exp\Par{V\sum_{s\in[t]}\eta_s^2}
    \Tr(\mc_t\mc_t^\top).
\]
The proof of Lemma~\ref{lem:expected_total_energy_vary_step} inductively shows a majorization relationship between the spectra of $\E[\mb_t\mb_t^\top]$ and a scaled population counterpart $\mc_t\mc_t^\top$, by using the von Neumann trace inequality and our matrix variance assumption to bound the effect of each increment. 

\textbf{Lower bound.} Finally, we complement Theorem~\ref{thm:gapfree_oja} with a lower bound in Theorem~\ref{thm:pca_lower}, which obtains matching dependences in all parameters up to polylogarithmic factors. Qualitatively similar lower bounds to Theorem~\ref{thm:gapfree_oja} (e.g., Theorem 32, \cite{GarberHJKMNS16}, and Theorem 6, \cite{allen2017first}) were already known, and our main contribution is to slightly strengthen the construction to hold for the entire range of $V$ and $\lam_1$. In particular, our proof builds upon the lower bound construction of \cite{allen2017first}.

\textbf{Gap-free probability boosting.} A standard strategy for boosting the success probability of PCA under an eigengap is to apply geometric aggregation (e.g., Lemma 3.10, \cite{kumar2024oja}). Unfortunately, a direct output aggregation fails in a gap-free setting: when the leading
eigenvalue has multiplicity, even two exact solutions may be orthogonal. Nonetheless, our proof strategy for Theorem~\ref{thm:gapfree_oja} proceeds by arguing constant probability closeness of each $\mb_t$ to the population matrix $\mc_t$, making it amenable to an intermediate geometric aggregation step. Our Algorithm~\ref{alg:boosted_sketch_oja} applies independent Oja products $\mb_t^{(r)}$ across $R = O(\log (\frac 1 \zeta))$ disjoint streams and
initializes each stream with the same Gaussian matrix $\mg$, using a slightly oversampled dimension (i.e., with $O(\log (\frac 1 \zeta))$ random vectors rather than a single vector). Together with standard results on the concentration of Gaussian traces, we show that we can aggregate these disjoint streams to a center compatible with the strategy in \eqref{eq:triangle_ineq_intro}, at a relatively mild $\polylog(\frac 1 \zeta)$ cost to the sample complexity and space overhead.

\textbf{Energy PCA.} We next consider an energy PCA guarantee for Oja's algorithm. A direct cPCA-to-ePCA reduction (e.g. Lemma 8, \cite{JambulapatiKLPPT24}) results in a suboptimal sample complexity scaling as $\frac1{\alpha^3}$ for an $\alpha$-ePCA guarantee. In Proposition~\ref{prop:multiscale_cpca_epca}, we consider a multiscale cPCA, with simultaneous guarantees on the projections to eigenvalues below a specified threshold $(1-u)\lambda_1$ for all choices of $u\in(0,1)$, as opposed to just $u = \gamma$. By integrating over $u$, we are able to obtain a sample complexity scaling as in $\frac1{\alpha^2}$ in Theorem~\ref{thm:gapfree_epca}, which we also show is tight in Corollary~\ref{cor:epca_lower_bound}.

\textbf{Application to DP PCA.} Private PCA is a natural application of Theorem~\ref{thm:gapfree_oja}. The DP-PCA method of \cite{LiuKJO22} forms minibatch covariance estimates and adds Gaussian perturbations, so its effective Oja updates are general matrix-valued rather than rank one. Their utility analysis invokes the gapped Oja guarantee of \cite{jain2016streaming}, and consequently depends on $\lam_1-\lam_2$.

Compared to \cite{LiuKJO22}, our analysis also yields an improved dependence on the gap parameter $\gamma$, set to $1 - \frac{\lam_2}{\lam_1}$ in their setting. We reuse the full dataset at every Oja step, rather than splitting it into fresh minibatches as in \cite{LiuKJO22}. This choice improves the sensitivity of each update by a factor of $b$, where $b$ is the number of mini-batches, while leading to $b$ passes over each sample. By paying for these passes using advanced composition (or R\'enyi DP \cite{Mironov17}, to give slightly tighter guarantees), this only incurs an $\approx \sqrt{b}$ overhead, the source of our savings.

Interestingly, our analysis directly uses the algorithm's privacy to argue about its correctness. This need arises due to a dependency between a currently estimated subspace and the data, which would affect clipping thresholds. We instead use a near-independence guarantee implied by DP to save a $\poly(d)$ factor in the threshold magnitude, which directly reflects in our sample complexity.

\subsection{Related work}
\label{ssec:intro_related}

\textbf{Streaming and gap-free PCA.}
Finite-sample analyses of streaming PCA include incremental PCA~\cite{BalsubramaniDF13}, memory-optimal block methods~\cite{MitliagkasCJ13}, and stochastic power or matrix-factorization methods~\cite{DeSaOR15}. Other variants address Markovian data~\cite{KumarS23}, sparse leading eigenvectors~\cite{kumar2024oja}, entrywise uncertainty quantification~\cite{KumarPS25}, and low-precision computation~\cite{DasguptaKPS25}. Most closely related to our work, \cite{jain2016streaming} obtained the first near-optimal gapped rates for general, possibly nonsymmetric matrix updates, while \cite{LiWLZ18} give near-optimal gapped guarantees for sub-Gaussian PCA. Relatedly, \cite{Shamir16} gives an early eigengap-free guarantee permitting general PSD stochastic matrices, but with a slower objective rate and low success probability from random initialization. Later, \cite{allen2017first} established an efficient global near-optimal gap-free analysis for rank-one streaming $k$-PCA, and \cite{Liang23} obtained sharp gap-free rates for sub-Gaussian data, again requiring rank-one updates. In another direction, \cite{HuangNW21} extends nearly offline-optimal streaming-PCA guarantees to arbitrary-rank updates under an eigengap. 

We note that this work focuses on the $1$-PCA problem, i.e., approximating the top eigenvector of a population average $\msig$ from samples. We leave open the analogous question for $k$-PCA for $k > 1$, where a similar situation holds in the current literature: \cite{allen2017first} gave a gap-free result for $k$-PCA under rank-one updates, and \cite{HuangNW21} removed the rank restriction, but used an eigengap.

{\textbf{Noisy power methods.}
Under Model~\ref{model:oja_special} and an eigengap assumption $\lam_2(\msig) \le (1 - \gamma)\lam_1(\msig)$, \cite{HardtP14} gives a suboptimal sample complexity scaling as $\widetilde O(\frac{V}{\lam_1^2\gamma^3\Delta} + \frac 1 \gamma)$ for minibatched stochastic matrix-vector products, even before accounting for their additional projected-noise condition (see the statement of their Corollary 1.1). This incurs an extra factor of $\frac 1 \gamma$ in the leading term compared with Theorem~\ref{thm:gapfree_oja}. Later, \cite{BalcanDWY16} replaces a dependence on $\lambda_1-\lambda_2$ by $\lambda_1-\lambda_{q+1}$, but requires maintaining at least $q$ directions. Notably, both results hold only in the gapped setting.
}

{\textbf{Differentially private PCA.}
For arbitrary row-bounded datasets, early approaches sample a direction using the exponential mechanism~\cite{ChaudhuriSS13}, while Analyze Gauss~\cite{DworkTTZ14} adds a symmetric Gaussian matrix to the empirical covariance and then extracts its leading eigenspace. 
Notably, when applying such results to i.i.d.\ sub-Gaussian data, the resulting sample complexity is at least $d^{1.5}$ up to logarithmic factors. 
Specializing to i.i.d.\ statistical models, \cite{LiuKO22} use robust one-dimensional scores within a propose-test-release framework to obtain nearly information-theoretically optimal private PCA under sub-Gaussian and hypercontractive assumptions, although the resulting estimator is not computationally efficient. The black-box reduction of \cite{HopkinsKMN23} converts suitable robust estimators into private mean and covariance estimators, from which PCA can be obtained by post-processing when covariance error controls the desired subspace. Closest to our algorithm, \cite{LiuKJO22} give a single-pass minibatched Oja method with nearly optimal rates for sub-Gaussian data under an eigengap. Subsequent specialized results obtain minimax rates for rank-$r$ spiked covariance models~\cite{CaiXZ24} and robustness to heavy tails and contamination under elliptical models~\cite{KimJ25}. 
}

\section{Preliminaries}

In Section~\ref{ssec:notation}, we give basic notation used throughout the paper, and in Section~\ref{ssec:cpca}, we state the main streaming PCA problem we consider. In Section~\ref{ssec:bernstein}, we state a baseline result in the offline setting via the matrix Bernstein theorem, under a slight strengthening of the problem formulation. We defer preliminaries on differential privacy, used in our main application, to Section~\ref{sec:dp_pca}.

\subsection{Notation}\label{ssec:notation}

We use $X \perp Y$ to denote that random variables $X$ and $Y$ are independent. We use $\ind_{\calE}$ to denote the $0$-$1$ indicator random variable of an event $\calE$. For two measures $\pi$, $\mu$ over the same sample space $\Omega$, which we identify with corresponding distributions, $\TV(\pi, \mu) \defeq \half \int_\Omega |\pi - \mu|\dd\omega$ denotes their TV distance and $\KL(\pi \| \mu) \defeq \int \pi \log \frac \pi \mu \dd \omega$ denotes their KL divergence. 

Vectors are denoted in lowercase boldface and matrices in uppercase boldface. We use $\0_d$ and $\1_d$ to denote the all-zeroes and all-ones vectors in $\R^d$, $\id_d$ to denote the $d \times d$ identity matrix, and $\0_{m \times n}$ to denote the $m \times n$ all-zeroes matrix. We use $[d]$ to denote $\{i \in \N: 1 \le i \le d\}$. For $p \ge 1$ including $p = \infty$ we use $\norm{\cdot}_p$ to denote the $\ell_p$ norm of a vector, and $\norm{\cdot}_{S_p}$ to denote the Schatten-$p$ norm of a matrix. The set $\Sym^{d \times d}$ denotes all $d \times d$ symmetric matrices, and $\PSD^{d \times d}$ denotes the subset of positive semidefinite matrices. We use $\calN(\vm, \msig)$ to denote the multivariate Gaussian with mean $\vm \in \R^d$ and covariance $\msig \in \PSD^{d \times d}$.
We use $\normop{\cdot}$ to denote the ($\ell_2$ induced) operator norm of a matrix, and $\normf{\cdot}$ to denote its Frobenius norm, i.e., Schatten-$2$ norm. We use $\lam_i(\cdot)$ to denote the $i^{\text{th}}$ largest eigenvalue of a symmetric matrix, and $\Tr(\cdot)$ for the trace. We say a matrix is orthonormal if its columns $\{\vu_i\}$ satisfy $\inprod{\vu_i}{\vu_j} = \ind_{i = j}$. For unit vectors $\vu, \vv$ we define
\[\msign\Par{\vu, \vv} \defeq \min\Brace{\norm{\vu - \vv}_2, \norm{\vu + \vv}_2}.\]

\begin{lemma}\label{lem:msign_triangle}
$\msign$ satisfies the triangle inequality.
\end{lemma}
\begin{proof}
For unit $\vu, \vv, \vw$, if $\msign\Par{\vu, \vw} = \norm{\vu - \sigma\vw}_2$, $\msign\Par{\vw, \vv} = \norm{\vw - \tau\vv}_2$, for $(\sigma, \tau) \in \{\pm 1\}^2$,
\[\msign\Par{\vu, \vv} \le \norm{\vu - \sigma\tau \vv}_2 \le \norm{\vu - \sigma \vw}_2 + \norm{\sigma \vw - \sigma\tau\vv}_2 = \msign\Par{\vu, \vw} + \msign\Par{\vw, \vv}.\]
\end{proof}

\subsection{Main problem}\label{ssec:cpca}

To state our main problem, we recall the following helpful definition from \cite{JambulapatiKLPPT24}, which has emerged as a useful gap-free notion of PCA in the literature \cite{GarberH15, ZhuL16, allen2017first}.

\begin{definition}[cPCA]\label{def:cpca}
Let $(\gamma, \Delta) \in (0, 1)^2$, and let $\msig \in \Sym^{d \times d}$. We say that a unit vector $\vv \in \R^{d}$ is a \emph{$(\gamma,\Delta)$-cPCA} (correlation PCA) of $\msig$ if, letting orthonormal $\ms \in \R^{d \times r}$ have the same column span as the eigenspace of $\msig$ corresponding to eigenvalues $< (1 - \gamma) \lam_1(\msig)$,
\[\normf{\ms^\top \mv}^2 \le \Delta.\]
\end{definition}

We now state the main statistical model we consider in this paper.

\begin{model}\label{model:oja_special}
Fix $\lam_1 > 0$ and $V > 0$.
Let $\{\ma_t \in \R^{d \times d}\}_{t \in [n]}$ be i.i.d.\ with $\E \ma_t = \msig \in \PSD^{d \times d}$, and
\[\normop{\msig} = \lam_1,\quad \max\Brace{\normop{\E\Brack{\Par{\msig - \ma_t}\Par{\msig - \ma_t}^\top}}, \normop{\E\Brack{\Par{\msig - \ma_t}^\top\Par{\msig - \ma_t}}}}\le V.\]
\end{model}

The main problem this paper focuses on is computing a cPCA of $\msig$, given access to $\{\ma_t\}_{t \in [n]}$ arising from Model~\ref{model:oja_special}. Our algorithms' sample complexities will depend on five parameters: $(V, \lam_1)$ from Model~\ref{model:oja_special}, $(\gamma, \Delta)$ from Definition~\ref{def:cpca}, and the failure probability, denoted $\zeta \in (0, 1)$.

We consider this problem in two settings: the batch setting where one can arbitrarily manipulate the $\{\ma_t\}_{t \in [n]}$, and the streaming setting, our main focus. In the streaming setting, the $\{\ma_t\}_{t \in [n]}$ are given in a stream, and once we receive $\ma_t$ we can perform an update and then it is discarded from memory. The goal is to solve the cPCA problem with low external memory, ideally $O(d)$. 

The main algorithm we consider for streaming PCA is Oja's algorithm (Algorithm~\ref{alg:oja}).

\begin{algorithm}[ht]
\DontPrintSemicolon 
\caption{$\Oja(\{\ma_t, \eta_t\}_{t \in [n]})$}\label{alg:oja} 
\textbf{Input:} $\{\ma_t \in \R^{d \times d}, \eta_t > 0\}_{t \in [n]}$ \;
$\vw_0 \gets \vg/\norm{\vg}_2$, where $\vg \sim \Nor(\0_d, \id_d)$\; 
\For{$t \in [n]$}{
$\vw_t \gets (\id_d + \eta_t \ma_t) \vw_{t - 1}$\;
$\vw_t \gets \vw_t/\norm{\vw_t}_2$\;
}
\Return $\vw_n$ 
\end{algorithm}

Note that Algorithm~\ref{alg:oja} is clearly a streaming algorithm. We introduce some helpful notation:
\begin{equation}\label{eq:bi_def}\begin{aligned}
    \mb_t &\defeq \Par{\id_d + \eta_t \ma_t}\cdots\Par{\id_d + \eta_1 \ma_1},\\
\mc_t &\defeq (\mi_d+\eta_t\msig)\cdots(\mi_d+\eta_1\msig).
\end{aligned}
\end{equation}
With this notation, the updates in Algorithm~\ref{alg:oja} are equivalent to $\vw_t \gets \frac{\mb_t \vg}{\norm{\mb_t \vg}_2}$. Further, $\mc_t$ helps track the unnormalized iterates for the corresponding updates using the average matrix $\msig$.

\subsection{Baseline via matrix Bernstein}\label{ssec:bernstein}

As a baseline, we recall a folklore result that in the batch setting, any approximate cPCA of the empirical covariance (with appropriate parameters) also solves the statistical cPCA problem. This result is stated under a slight strengthening of Model~\ref{model:oja_special} that imposes a probability $1$ bound on each sample $\ma_t$, but naturally yields a high-probability guarantee unlike Theorem~\ref{thm:gapfree_oja}. 

\begin{model}\label{model:oja_prob1}
Fix $\lam_1 > 0$, $V > 0$, and $M > 0$.
Let $\{\ma_t\}_{t \in [n]}$ from Model~\ref{model:oja_special} additionally satisfy
\[\normop{\msig - \ma_t} \le M \text{ with probability } 1.\]
\end{model}

\begin{proposition}[Gap-free PCA via matrix Bernstein]\label{prop:gapfree_offline}
Under Model~\ref{model:oja_prob1}, let $\hmsig \defeq \frac 1 {2n} \sum_{t \in [n]} (\ma_t + \ma_t^\top)$, and let $\vv$ be any $(\frac \gamma 6, \frac \Delta 4)$-cPCA for $\hmsig$. Then for any $\zeta \in (0, \frac 1 3)$, $\vv$ is a $(\gamma, \Delta)$-cPCA for $\msig$ with probability $\ge 1 - \zeta$, if for an appropriate constant,
\[n = \Omega\Par{\Par{\frac{V}{\lam_1^2\gamma^2\Delta} + \frac{M}{\lam_1\gamma\sqrt{\Delta}}}\log\Par{ \frac d \zeta}}. \]
\end{proposition}
\begin{proof}
This is almost the statement of Proposition 1, \cite{Tian26}, up to the assumptions on $\{\ma_t\}_{t \in [n]}$. The proof of Proposition 1, \cite{Tian26} shows the result if with probability $\ge 1 - \zeta$,
\[\normop{\hmsig - \msig} \le \frac{\lam_1\gamma \sqrt \Delta}{4}.\]
To show this, set $\me_t \defeq \ma_t - \msig$ and $\me \defeq \frac 1 n \sum_{t \in [n]} \me_t$. The $\me_t$ are independent and mean zero, so the bounds from Model~\ref{model:oja_prob1} and the matrix Bernstein inequality (Theorem 6.6.1, \cite{tropp2015introduction}) prove the above bound on $\normsop{\me}$. The claim follows from Jensen's inequality and $\hmsig - \msig = \half(\me + \me^\top)$.
\end{proof}
\section{Oja's Algorithm}
\label{sec:oja}

In this section we give a new analysis of Oja's algorithm (Algorithm~\ref{alg:oja}), yielding cPCA guarantees in the general gap-free setting of Model~\ref{model:oja_special}. 
To simplify notation, we let orthonormal $\ms$ span the eigenspace of $\msig$ corresponding to eigenvalues $< (1 - \gamma) \lam_1(\msig)$ (in line with Definition~\ref{def:cpca}), and denote the associated orthogonal projector by $\mpp \defeq \ms\ms^\top$.
Also, we require one helper fact.

\begin{fact}\label{fact:abel_majorize}
Let $\vx, \vy, \vc \in \R^d_{\ge 0}$ have nonincreasing coordinates, and suppose $\vy$ weakly majorizes $\vx$. Then
\[\sum_{j \in [k]} \vc_j\vx_j \le \sum_{j \in [k]} \vc_j \vy_j \text{ for all } k \in [d].\]
\end{fact}
\begin{proof}
For all $k \in [d]$ let $\vs_k \defeq \sum_{j \in [k]}\vx_j$ and $\mathbf{t}_k \defeq \sum_{j \in [k]}\vy_j$. Then Abel's summation formula gives
\[\sum_{j \in [k]} \vc_j (\vy_j - \vx_j) = \vc_k (\mathbf{t}_k - \vs_k) + \sum_{j \in [k - 1]} (\vc_j - \vc_{j + 1})(\mathbf{t}_j - \vs_j) \ge 0.\]
\end{proof}

We first analyze a one-step power method, analogously to Lemma 3.1, \cite{jain2016streaming}.

\begin{lemma}
\label{lem:gapfree_one_step_power}
Let $\zeta \in (0, \frac 1 3)$, let $\mb\in\R^{d\times d}$ not be the all-zeroes matrix, and let $\mv \in \R^{d \times r}$ be orthonormal. If $\vg \sim \calN(\0_d, \id_d)$, then with probability $\ge 1 - \zeta$, 
\bas{
\frac{\norm{\mv^\top \mb \vg}_2^2}{\norm{\mb \vg}_2^2} \le \frac{90\log \Par{\frac 1 \zeta}}{\zeta^2} \cdot \frac{\normf{\mv\mv^\top \mb}^2}{\normf{\mb}^2}.
}
\end{lemma}

\begin{proof}
Define
    $\mh\defeq\mb^\top\mb$ and $\mk\defeq\mb^\top \mv\mv^\top\mb$.
Then our goal is to bound
\bas{
    \frac{\norm{\mv^\top \mb \vg}_2^2}{\norm{\mb \vg}_2^2} 
    =
    \frac{\vg^\top\mk\vg}{\vg^\top\mh\vg}.
}
For the denominator, standard Gaussian anti-concentration (e.g., Lemma A.2.1, \cite{kumar2024oja}) shows
\bas{
    \vg^\top\mh\vg
    \geq
    \frac{\zeta^2}{4e}\Tr(\mh)
}
with probability at least $1 - \frac \zeta 2$.
Similarly, for the numerator, standard $\chi^2$ concentration bounds (e.g., Lemma 1, \cite{LaurentM00}) along with $\Tr(\mk^2) \le \Tr(\mk)^2$, $\normsop{\mk} \le \Tr(\mk)$, gives
\begin{align*}
\Pr\Par{\vg^\top \mk \vg > (1 + 2\sqrt t + 2t)\Tr(\mk)} \leq \exp(-t)\text{ for all } t > 0.
\end{align*}
Plugging in $t = \log (\frac 2\zeta )\le 2 \log (\frac 1 \zeta)$ and combining the above three displays gives the result.
\end{proof}

In Lemmas~\ref{lem:expected_total_energy_vary_step} and~\ref{lem:gapfree_bad_trace_vary_step}, we derive bounds on the ratio in Lemma~\ref{lem:gapfree_one_step_power} for $\mv \gets \ms$, as $\mb \gets \mb_t$ undergoes the updates of Algorithm~\ref{alg:oja}. We begin by tracking a trace-based potential.

\begin{lemma}\label{lem:expected_total_energy_vary_step}
Under Model~\ref{model:oja_special} and notation \eqref{eq:bi_def}, the iterates of Algorithm~\ref{alg:oja} satisfy, for all $t \in [n]$,
    \bas{
    \E\Brack{\Tr\Par{\mb_t\mb_t^\top}}
    \leq  \exp\Par{V\sum_{s\in[t]} \eta_s^2}\Tr\Par{\mc_t\mc_t^\top}.
}
\end{lemma}
\begin{proof}
Let $\lam_j \defeq \lam_j(\msig)$ for shorthand. We prove inductively that, for every $s\ge 0$ and $k \in[d]$, \begin{equation}\label{eq:induct_num}
\sum_{j\in[k]} \lambda_j\Par{\E[\mb_s\mb_s^\top]}\le \Par{\prod_{r \in [s]}(1+V\eta_r^2)}\sum_{j\in[k]}\prod_{r\in [s]}(1+\eta_r\lambda_j)^2.
\end{equation}

Then, taking $s \gets t$, $k \gets d$, and using $1+x\le e^x$ proves the claim, since all $\id_d + \eta_s \msig$ commute.

Clearly \eqref{eq:induct_num} holds for $s = 0$ (where we take empty products as $1$).
For the inductive step, suppose that \eqref{eq:induct_num} holds for $s - 1$ and all $k \in [d]$. Upon expanding, we have
\begin{equation}\label{eq:two_terms_vary}
\begin{aligned}
\E[\mb_s\mb_s^\top]
=
(\id_d+\eta_s\msig)\E[\mb_{s-1}\mb_{s-1}^\top](\id_d+\eta_s\msig) +
\eta_s^2\E\left[
    (\ma_s-\msig)\E[\mb_{s-1}\mb_{s-1}^\top](\ma_s-\msig)^\top
\right].
\end{aligned}
\end{equation}

We bound the two terms separately in order to apply \eqref{eq:induct_num}. For the first term, for any PSD $\mc, \mh$, letting $\mq$ be the projector onto any top-$k$ eigenspace of $\mc\mh\mc$, and using $\mc\mq\mc \preceq \mc^2$,
\begin{align*}
\sum_{j\in[k]}\lambda_j(\mc\mh\mc) &= \inprod{\mq}{\mc\mh\mc} =  \Tr(\mh \mc\mq\mc) \\
&\le \sum_{j \in [k]} \lam_j(\mh)\lam_j(\mc\mq\mc)
\le \sum_{j\in[k]}\lambda_j(\mc)^2\lambda_j(\mh).
\end{align*}
For the second term, observe that for every rank-$k$ orthogonal projector $\mq$,
\[\0_{d\times d} \preceq
\E[(\ma_s-\msig)^\top \mq(\ma_s-\msig)]
\preceq V\id_d,\quad \Tr\Par{\E\Brack{(\ma_s-\msig)^\top\mq(\ma_s-\msig)}} \le kV.\]
Then by the von Neumann trace inequality,
\begin{equation*}
\Tr\left(
    \mq\E[(\ma_s-\msig)\mh(\ma_s-\msig)^\top]
\right)
=
\Tr\left(
    \mh\E[(\ma_s-\msig)^\top \mq(\ma_s-\msig)]
\right)                                                     \leq
V\sum_{j=1}^k\lambda_j(\mh).
\end{equation*}
By supremizing this over rank-$k$ projectors $\mq$, for every $\mh\in \PSD^{d \times d}$ and $k\in[d]$,
\bas{
    \sum_{j\in[k]}\lambda_j\!\left(
\E[(\ma_s-\msig)\mh(\ma_s-\msig)^\top]
\right)
\leq V\sum_{j\in[k]}\lambda_j(\mh),
    }
which bounds the second term. Combining the above displays into \eqref{eq:two_terms_vary}, and using the triangle inequality of the Ky Fan norm,
\begin{align*}
    \sum_{j\in[k]}\lambda_j(\E[\mb_s\mb_s^\top]) &\le \sum_{j\in[k]}\Par{(1+\eta_s\lambda_j)^2+V\eta_s^2}\lambda_j(\E[\mb_{s-1}\mb_{s-1}^\top])\\
    &\le (1+V\eta_s^2)\sum_{j\in[k]}(1+\eta_s\lambda_j)^2\lambda_j(\E[\mb_{s-1}\mb_{s-1}^\top])\\
    &\le \Par{\prod_{r\in[s]}(1+V\eta_r^2)}\sum_{j\in[k]}\prod_{r\in[s]}(1+\eta_r\lambda_j)^2
\end{align*}
where the third line applies Fact~\ref{fact:abel_majorize}
with 
\begin{align*}
 \vc_j \gets (1+\eta_s\lambda_j)^2, \quad \vx_j \gets \lambda_j(\E[\mb_{s-1}\mb_{s-1}^\top]), \quad \vy_j \gets \Par{\prod_{r\in[s-1]}(1+V\eta_r^2)}\prod_{r\in[s-1]}(1+\eta_r\lambda_j)^2,
\end{align*}
where $\vy$ weakly majorizes $\vx$ is the inductive hypothesis.
Thus \eqref{eq:induct_num} holds as desired.
\end{proof}

\begin{lemma}
\label{lem:gapfree_bad_trace_vary_step}
Under Model~\ref{model:oja_special} and notation \eqref{eq:bi_def}, let $\zeta \in (0, \frac 1 3)$, let $t \in [n]$, and suppose
$\eta_s\le \frac 1 {\lam_1}$ for all $s \in [t]$. If we let $q_t \defeq V\sum_{s \in [t]} \eta_s^2$, and $q_t \le \frac \zeta 2$, then with probability $\ge 1 - \zeta$,
\bas{
\frac{\normf{\mpp \mb_t}^2}{\normf{\mb_t}^2}
\leq 2d\exp\Par{-\gamma\lam_1\sum_{s \in [t]} \eta_s} + \frac{8\Par{\exp(q_t) - 1}}{\zeta}.
}
\end{lemma}

\begin{proof}
First, observe that since $\normf{\mc_t}^2 \ge \normsop{\mc_t^2} = \prod_{s\in[t]}(1+\eta_s\lambda_1)^2$,
$\mpp$ commutes with all of the $\id_d + \eta_s \msig$, and the corresponding eigenvalues of $\msig$ are $\le (1 - \gamma) \lam_1$,
\begin{equation}\label{eq:c_ratio_bound}
\frac{\normf{\mpp\mc_t}^2}{\normf{\mc_t}^2}
\le
d\prod_{s\in[t]}
\left(
    \frac{1+\eta_s(1-\gamma)\lambda_1}
         {1+\eta_s\lambda_1}
\right)^2
\le
d\exp\left(
    -\gamma\lambda_1\sum_{s\in[t]}\eta_s
\right).
\end{equation}
Next, independence of the stream in Model~\ref{model:oja_special} shows that $\E[\mb_t]=\mc_t$, so Lemma~\ref{lem:expected_total_energy_vary_step} gives
\begin{equation}\label{eq:bc_compare}
\E \normf{\mb_t - \mc_t}^2 =
\E\normf{\mb_t}^2 - 
\normf{\mc_t}^2 \le \Par{\exp\Par{q_t} - 1} \normf{\mc_t}^2.
\end{equation}
Thus, by Markov's inequality, we have with probability $\ge 1 - \zeta$ that
\begin{equation}\label{eq:good_event_bc}\frac{\normf{\mb_t - \mc_t}}{\normf{\mc_t}} \le \sqrt{\frac{\exp(q_t) - 1}{\zeta}} < 1, \end{equation}
where the last inequality used our assumption $q_t \le \frac \zeta 2$. Under this event, we have $\mb_t \neq \0_{d \times d}$ by the triangle inequality. Further, for non-zero $\mx, \my$,
\begin{equation}\label{eq:crazy_triangle_ineq}
    \left\|
        \frac{\mx}{\normf{\mx}}-
        \frac{\my}{\normf{\my}}
    \right\|_{\rm F} \le \normf{\frac{\mx - \my}{\normf{\my}}} + \normf{\mx \cdot \Par{\frac 1 {\normf{\mx}} - \frac 1 {\normf{\my}}}}
    \leq
    2\frac{\normf{\mx-\my}}{\normf{\my}}.
\end{equation}
Finally, applying  \eqref{eq:crazy_triangle_ineq} with $(\mx, \my) \gets (\mb_t, \mc_t)$ implies
\begin{align*}
\frac{\normf{\mpp \mb_t}}{\normf{\mb_t}} &\le \normf{\mpp \frac{\mc_t}{\normf{\mc_t}}} + \normf{\mpp\Par{\frac{\mb_t}{\normf{\mb_t}} - \frac{\mc_t}{\normf{\mc_t}}}} \\
&\le \frac{\normf{\mpp \mc_t}}{\normf{\mc_t}} + \normf{\frac{\mb_t}{\normf{\mb_t}} - \frac{\mc_t}{\normf{\mc_t}}} \\
&\le \sqrt{d\exp\left(
    -\gamma\lambda_1\sum_{s\in[t]}\eta_s
\right)} + 2\sqrt{\frac{\exp(q_t) - 1}{\zeta}},
\end{align*}
where we used \eqref{eq:c_ratio_bound} and \eqref{eq:good_event_bc} in the last line. The conclusion follows from $(a + b)^2 \le 2(a^2 + b^2)$.
\end{proof}

At this point, we are ready to give our main bound in this section.
\begin{theorem}\label{thm:gapfree_oja}
Let $\zeta \in (0, \frac 1 3)$ and $(\gamma, \Delta) \in (0, 1)^2$. Under Model~\ref{model:oja_special}, if
\[n = \Omega\Par{\frac{V}{\lam_1^2\gamma^2 \Delta} \cdot \frac{\log^2\Par{\frac d {\Delta\zeta}}\log\Par{\frac 1 \zeta}}{\zeta^3}  + \frac{\log\Par{\frac d {\Delta\zeta}}}{\gamma}},\]
for an appropriate constant,
then there exists $\beta\in \R_{>0}$ such that
taking $\eta_t = \frac{\log(\frac {1440d} {\Delta\zeta^3})}{\gamma\lam_1(\beta + t)}$ for $t \in [n]$, the output of Algorithm~\ref{alg:oja} is a $(\gamma,\Delta)$-cPCA of $\msig$ with probability $\ge 1 - \zeta$.
\end{theorem}
\begin{proof}
For shorthand, denote $L \defeq \log(\frac{1440 d}{\Delta \zeta^3})$, and for a large enough constant $C$, let \[
\beta \defeq \max\Brace{\frac{8L}{\gam}, \frac{CVL^2}{\gam^2\lambda_1^2\Delta} \cdot \frac{\log\Par{\frac 1 \zeta}}{\zeta^3}},
\]
and $n\ge 4\beta$, so that all $\eta_s \le \frac 1 {\lam_1}$. Moreover, $\sum_{s\in[n]}\frac{1}{(\beta+s)^2}\le \int_\beta^\infty \frac{\dd x}{x^2}=\frac{1}{\beta}$, so for large enough $C$,
\[
q_n = V\sum_{s\in[n]}\eta_s^2 \le \frac{VL^2}{\gam^2\lambda_1^2 \beta}\le \frac{\Delta} 2 \cdot \frac{\zeta^3}{32 \cdot 360 \cdot \log \Par{\frac 2 \zeta}}.
\]
Thus, $q_n \le \frac \zeta 4$, so Lemmas~\ref{lem:gapfree_one_step_power} and~\ref{lem:gapfree_bad_trace_vary_step} both apply at failure probability $\frac \zeta 2$, and combining gives
\[\frac{\norm{\ms^\top \mb_n \vg}_2^2}{\norm{\mb_n \vg}_2^2} \le \frac{360\log\Par{\frac 2 \zeta}}{\zeta^2} \Par{2d\exp\Par{-\gamma\lam_1\sum_{s \in [n]} \eta_s} + \frac{16(\exp(q_n) - 1)}{\zeta}}, \]
with probability $\ge 1 - \zeta$ over the randomness of $\vg \sim \calN(\0_d, \id_d)$ and Model~\ref{model:oja_special}. Condition on this event henceforth. For the first term above, since $n \ge 4\beta$, an integral comparison gives
\[\gamma \lam_1 \sum_{s \in [n]}\eta_s = L\sum_{s\in[n]}\frac{1}{\beta + s}\ge  L\int_{1}^{n + 1}\frac{\dd x}{\beta + x}  \ge L\log\Par{\frac{\beta+n+1}{\beta + 1}} \ge L.\]
Combining the above three displays concludes the proof, upon simplifying using $\exp(q_n) - 1 \le 2q_n$, and $\log(\frac 2 \zeta) \le \frac 1 \zeta$, in the relevant parameter regimes.
\end{proof}

\section{PCA Lower Bound}
\label{sec:lower_bound}

We give an information-theoretic lower bound that shows the leading-order parameter dependence in Theorem~\ref{thm:gapfree_oja} is qualitatively tight, for any choice of $V, \lam_1$. We state our result under Model~\ref{model:oja_special}, but our hard instance is even more well-behaved: the matrices $\ma_i$ are always PSD.

To begin, we require a standard formulation of Le Cam's two point method.

\begin{lemma}[Theorem 2.2(i), \cite{Tsybakov09}]\label{lem:lecam}
Let $P_0, P_1$ be probability distributions on the same measurable space $\Omega$, and let $\phi: \Omega \to \{0, 1\}$ be a (possibly randomized) function. Then,
\[\Pr_{\omega \sim P_0}\Brack{\phi(\omega) = 1} + \Pr_{\omega \sim P_1}\Brack{\phi(\omega) = 0} \ge 1 - \TV\Par{P_0, P_1}.\]
\end{lemma}

We can now state and prove our lower bound.

\begin{theorem}\label{thm:pca_lower}
Fix any choice of $\lam_1 > 0$, $V > 0$, $\gamma \in (0, \frac 1 4)$, and $\Delta \in (0, \frac 1 4)$. There is no algorithm $\calA$ that takes as input $\{\ma_i\}_{i \in [n]}$ from Model~\ref{model:oja_special}, and outputs a $(\gamma, \Delta)$-cPCA of $\msig$ with probability $\ge \frac 2 3$, even assuming that $\ma_i \in \PSD^{d \times d}$ for all $i \in [n]$, unless for an appropriate constant,
\[n = \Omega\Par{\frac{V}{\lam_1^2 \gamma^2 \Delta}}.\]
\end{theorem}
\begin{proof}
We begin by defining matrices used in our construction. Let $\alpha \defeq 1.5\arcsin\sqrt{\Delta}$, and
\[\vu_0 \defeq \begin{pmatrix} \cos \alpha \\ -\sin \alpha \end{pmatrix},\quad \vv_0 \defeq \begin{pmatrix} \sin \alpha \\ \cos \alpha \end{pmatrix},\quad \vu_1 \defeq \begin{pmatrix} \cos\alpha \\ \sin\alpha \end{pmatrix},\quad \vv_1 \defeq \begin{pmatrix} -\sin\alpha \\ \cos\alpha \end{pmatrix}.\]
Also, let
\[\msig_0 \defeq \lam_1 \vu_0\vu_0^\top + (1 - 2\gamma) \lam_1 \vv_0\vv_0^\top,\quad \msig_1 \defeq \lam_1 \vu_1\vu_1^\top + (1 - 2\gamma)\lam_1 \vv_1\vv_1^\top.\]
If $\vu$ is a $(\gamma, \Delta)$-cPCA of $\msig_0$, then $\msign(\vu, \vu_0) \le 2\sin(\frac \alpha 3)$, and a similar bound holds for $\msig_1$. We observe two helpful reformulations of $\msig_0, \msig_1$ used in our constructions. First,
\begin{equation}\label{eq:alt_sig_def}
\begin{gathered}
\msig_i = \md + (2i - 1)\beta \mh\text{ for } i \in \{0, 1\},\text{ where } \beta \defeq 2\gamma \lam_1 \sin(\alpha)\cos(\alpha),\\
\md \defeq \lam_1 \begin{pmatrix} \cos^2 \alpha + (1 - 2\gamma)\sin^2 \alpha & 0 \\ 0 & \sin^2 \alpha + (1 - 2\gamma)\cos^2 \alpha\end{pmatrix},\quad\mh \defeq \begin{pmatrix} 0 & 1 \\ 1 & 0 \end{pmatrix}.
\end{gathered}
\end{equation}
Second, letting the diagonal elements of $\md$ be $d_1$ and $d_2$, and $s \defeq d_1 + d_2 = 2(1-\gamma)\lam_1$,
\begin{equation}\label{eq:alt_sig_def_2}
\begin{gathered}
\msig_i = s\Par{\Par{\half + (2i - 1)\eta}\vw\vw^\top + \Par{\half - (2i - 1)\eta}\vz\vz^\top},\text{ where } \eta \defeq \frac{\beta}{2\sqrt{d_1d_2}}, \\
\vw \defeq \frac 1 {\sqrt s}\begin{pmatrix} \sqrt{d_1} \\ \sqrt{d_2} \end{pmatrix},\quad \vz \defeq \frac 1 {\sqrt s} \begin{pmatrix} \sqrt{d_1} \\ -\sqrt{d_2}\end{pmatrix}.
\end{gathered}
\end{equation}
Now suppose there is an algorithm $\calA$ as in the theorem statement, and define $\phi: (\PSD^{d \times d})^n \to \{0, 1\}$ as follows. Given matrices $\{\ma_i\}_{i \in [n]}$, let $\vu \defeq \calA(\{\ma_i\}_{i \in [n]})$ be the assumed algorithm's output. Then we let $\phi$ be the composition of $\calA$ with the map $\vu \to i \in \{0, 1\}$, where $i = 0$ if $\msign(\vu, \vu_0) \le \msign(\vu, \vu_1)$ and $i = 1$ otherwise. Because $\msign(\vu_0, \vu_1) = 2\sin \alpha > 4\sin(\frac \alpha 3)$, Lemma~\ref{lem:msign_triangle} implies that $\phi$ identifies $i \in \{0, 1\}$ whenever $\vu$ is a $(\gamma, \Delta)$-cPCA of the corresponding $\msig_i$.

We next define our distributions on the $\{\ma_i\}_{i \in [n]}$. We split into two cases depending on $V$. In each case, we define single-sample laws $p_0$, $p_1$ with means $\msig_0$, $\msig_1$, and let $P_0 \defeq p_0^{\otimes n}$, $P_1 \defeq p_1^{\otimes n}$ be the $n$-fold product laws. We show that our laws satisfy Model~\ref{model:oja_special}, and bound $\KL(p_0 \| p_1)$. In this setting, existence of the stated algorithm $\calA$ implies that
\begin{equation}\label{eq:test_error}\Pr_{\omega \sim P_0}\Brack{\phi(\omega) = 1} \le \frac 1 3,\quad \Pr_{\omega \sim P_1}\Brack{\phi(\omega) = 0} \le \frac 1 3.\end{equation}

\textit{Case 1: $V \le (1 - 2\gamma) \lam_1^2$.} We follow the notation \eqref{eq:alt_sig_def}. Let $t \defeq \sqrt{V + \beta^2}$. We define the law of $\ma \sim p_i$ for $i \in \{0, 1\}$ as follows. First, we draw $\sigma \in \{\pm 1\}$ with $\E[\sigma] = \frac{(2i - 1)\beta}{t}$, so that
\[\sigma = \begin{cases} 1 &\text{with probability } \half + \frac{(2i - 1)\beta}{2t} \\ 
-1 & \text{with probability } \half - \frac{(2i - 1)\beta}{2t} \end{cases},
\]
We then set $\ma \gets \md + \sigma t \mh$. Observe that $\E_{p_i}[\ma] = \md + (2i - 1)\beta \mh = \msig_i$ from \eqref{eq:alt_sig_def},
\[\E_{p_i}\Brack{\Par{\msig_i - \ma}^2} = \Par{t^2 - \beta^2} \mh^2 = V\id_2, \]
and $\ma$ has positive entries on the diagonal, with
\[\det\Par{\ma} = \det\Par{\md}- t^2 = (1 - 2\gamma)\lam_1^2 + \beta^2 - t^2 \ge 0.\]
Thus, draws from both $P_0$ and $P_1$ are valid instances of Model~\ref{model:oja_special}. We also have
\[\KL\Par{p_0 \| p_1} = \frac{\beta}{t} \log\Par{\frac{1 + \frac \beta t}{1 - \frac \beta t}} \le \frac{2\beta^2}{t^2 - \beta^2} \le \frac{18\gamma^2\lam_1^2\Delta}{V},\]
where we computed the KL between two distributions on $\{\pm 1\}$ with probabilities $\half \pm \frac{\beta}{2t}$, and used the bounds $\beta^2 \le 9\gamma^2 \lam_1^2 \Delta$ and $x \log \frac{1 + x}{1-x} \le \frac{2x^2}{1-x^2}$ valid for $x \in [0, 1)$.

\textit{Case 2: $V > (1 - 2\gamma) \lam_1^2$.} We follow the notation \eqref{eq:alt_sig_def}, \eqref{eq:alt_sig_def_2}. Let $R \defeq \lam_1 + \frac V {\lam_1}$ and $q \defeq \frac s R < 1$. We define the law of $\ma \sim p_i$ for $i \in \{0, 1\}$ as follows. We set
\begin{align*}
\ma = \begin{cases}
R\vw\vw^\top & \text{ with probability } q\Par{\half + (2i - 1)\eta} \\
R\vz\vz^\top & \text{ with probability } q\Par{\half - (2i - 1)\eta} \\
\0_{2\times 2}& \text{ with probability } 1- q
\end{cases}.
\end{align*}
Observe that $\E_{p_i}[\ma] = \msig_i$ from \eqref{eq:alt_sig_def_2}, and $\ma$ is clearly always PSD. Further, because $R\lam_1 - \lam_1^2 = V$ and $R(1 - 2\gamma)\lam_1 - (1 - 2\gamma)^2\lam_1^2 = (1 - 2\gamma)(V + 2\gamma\lam_1^2) \le V$, we have
\[\E_{p_i}\Brack{\Par{\msig_i - \ma}^2} = R \msig_i - \msig_i^2 \preceq V\id_2.\]
Thus, draws from both $P_0$ and $P_1$ are valid instances of Model~\ref{model:oja_special}. We can finally directly compute
\[\KL\Par{p_0 \| p_1} = (2\eta q) \log\frac{1 + 2\eta}{1 - 2\eta} \le \frac{8q\eta^2}{1 - 4\eta^2} = \frac{2q\beta^2}{d_1d_2 - \beta^2} \le \frac{8\beta^2}{V} \le \frac{72\gamma^2\lam_1^2\Delta}{V},\]
where we used $d_1d_2 - \beta^2 = (1 - 2\gamma)\lam_1^2$, $q \le \frac{2\lam_1}{R} \le \frac{2\lam_1^2}{V}$, and our earlier bound $\beta^2 \le 9\gamma^2\lam_1^2\Delta$.

In summary, in all regimes of $(V, \lam_1)$, there are instances $P_0, P_1$ of Model~\ref{model:oja_special} such that $P_0 = p_0^{\otimes n}$, $P_1 = p_1^{\otimes n}$, the mean of $p_i$ is $\msig_i$ for $i \in \{0, 1\}$, and $\KL(p_0 \| p_1) = O(\frac{\gamma^2 \lam_1^2\Delta}{V})$. 

By tensorization of KL divergence, we have for sufficiently small $n = o(\frac{V}{\lam_1^2 \gamma^2 \Delta})$ that
$\TV\Par{P_0, P_1} < \frac 1 3$.
Combining with Lemma~\ref{lem:lecam} and \eqref{eq:test_error}, this contradicts existence of $\calA$.
\end{proof}
\section{High-Probability Guarantees}
\label{sec:high_probability}

Theorem~\ref{thm:gapfree_oja} achieves a near-optimal rate of error as a function of the parameters in Model~\ref{model:oja_special} and Definition~\ref{def:cpca}, but only provides a low-confidence guarantee (i.e., with polynomial dependence on $\frac 1 \zeta$). We next give a confidence amplification procedure with only polylogarithmic overheads, without worsening the dependence on either of the cPCA parameters $(\gamma,\Delta)$.

Notably, standard confidence boosting procedures based on a direct geometric aggregation (e.g., Lemma 3.10, \cite{kumar2024oja}) do not work, since if the leading eigenvalue has multiplicity larger than one, even two exact cPCAs may be orthogonal. Instead, we aggregate sketches of the unnormalized Oja product, piggybacking off of closeness guarantees from Lemma~\ref{lem:gapfree_bad_trace_vary_step}. The main technical novelty in this section is that we require a polylogarithmic-dimension sketch (rather than the single Gaussian used in Algorithm~\ref{alg:oja}), so that the population-level sketch concentrates with high probability.

We begin with a standard Gaussian trace estimate.

\begin{lemma}
\label{lem:gaussian_trace_sketch}
Let $\mh\succeq\mzero$ be fixed and let $\mg\in\R^{d\times s}$ have i.i.d.\ $\Nor(0,1)$ entries. Then, 
\begin{equation}
\label{eq:gaussian_trace_sketch}
    \frac{s}{2}\Tr(\mh)
    \leq
    \Tr(\mg^\top\mh\mg)
    \leq
    2s\Tr(\mh),
\end{equation}
with probability $\ge 1-\exp(-\frac s {16})$ for each inequality.
\end{lemma}

\begin{proof}
Let the columns of $\mg$ be $\{\vg_i\}_{i \in [s]}$. Then, the middle term in \eqref{eq:gaussian_trace_sketch} is $X \defeq \sum_{j\in[s]}\vg_j^\top\mh\vg_j$. The claim follows from $\chi^2$ concentration. In particular, Lemma 1, \cite{LaurentM00} gives \begin{align*}
    \Pr\Par{X - s\Tr(\mh) \ge 2\sqrt{s\Tr(\mh^2)x}+2\normop{\mh}x} &\le \exp(-x)\\
    \Pr\Par{s\Tr(\mh) - X\ge 2\sqrt{s\Tr(\mh^2)x}}&\le \exp(-x).
\end{align*}
Using $x\gets \frac s {16}$, $\Tr(\mh^2)\leq\Tr(\mh)^2$, and $\normop{\mh}\leq\Tr(\mh)$ proves the claim.
\end{proof}

Next, we show that independent product sketches cluster around a common population sketch.

\begin{lemma}
\label{lem:product_sketch_concentration}
 Let $\mg\in\R^{d\times R}$ have i.i.d.\ $\Nor(0,1)$ entries. Under Model~\ref{model:oja_special}, if $\eta_s \le \frac 1 {\lam_1}$ for all $s \in [t]$ and $q_t \defeq V\sum_{s\in[t]}\eta_s^2\le \frac 1 {128}$, then with probability $\ge 1-3\exp(-\frac R {16})$ over $\mg$, the following hold. First,
\begin{equation}
\label{eq:population_sketch_bad_mass}
    \normf{\mpp\mz_{\star,t}}^2
    \leq
    4d\exp\Par{-\gam\lambda_1 \sum_{s\in[t]}\eta_s}, \text{ for } \mz_{\star,t}\defeq\frac{\mc_t\mg}{\normf{\mc_t\mg}}.
\end{equation}
Second, conditioning on $\mg$, over the randomness of Model~\ref{model:oja_special},
\begin{equation}
\label{eq:random_sketch_near_center}
    \Pr\Brack{
        \normf{
            \frac{\mb_t\mg}{\normf{\mb_t\mg}}-\mz_{\star,t}
        }
        \leq 8\sqrt{2q_t}}
        \geq \frac34.
\end{equation}
\end{lemma}
\begin{proof}
Throughout the proof, let $\mh_t\defeq\E[(\mb_t-\mc_t)^\top(\mb_t-\mc_t)]$. Then Lemma~\ref{lem:gaussian_trace_sketch} implies
\begin{equation}\label{eq:good_events_whp}\normf{\mc_t\mg}^2 \ge \frac R 2 \normf{\mc_t}^2,\quad \normf{\mpp \mc_t\mg}^2 \le 2R\normf{\mpp\mc_t}^2,\quad \Tr\Par{\mg^\top\mh_t\mg} \le 2R\Tr\Par{\mh_t},\end{equation}
all hold with the requisite failure probability. Condition on these events henceforth. Combining the first two events in \eqref{eq:good_events_whp} with \eqref{eq:c_ratio_bound} then immediately implies our first claim, \eqref{eq:population_sketch_bad_mass}.

Next, recall from \eqref{eq:bc_compare} and our range on $q_t$ that
\[\Tr\Par{\mh_t} = \E\normf{\mb_t - \mc_t}^2 \le 2q_t\normf{\mc_t}^2.\]
Combining with the first and third events in \eqref{eq:good_events_whp} then gives
\[\E\normf{\Par{\mb_t - \mc_t}\mg}^2 = \Tr\Par{\mg^\top \mh_t \mg} \le 2R\Tr(\mh_t) \le 4q_t R \normf{\mc_t}^2 \le 8q_t  \normf{\mc_t \mg}^2. \]
Thus, by Markov's inequality, with probability $\ge \frac 3 4$ over Model~\ref{model:oja_special},
\[\frac{\normf{(\mb_t - \mc_t)\mg}}{\normf{\mc_t \mg}} \le 4\sqrt{2q_t} \le \half.\]
Thus we have $\mb_t \mg \neq \0_{d \times R}$ on this event, so again applying \eqref{eq:crazy_triangle_ineq} (with $(\mx, \my) \gets (\mb_t\mg, \mc_t\mg)$), and using the definition of $\mz_{\star, t}$, concludes the proof of \eqref{eq:random_sketch_near_center}.
\end{proof}

We give our full high-probability method in Algorithm~\ref{alg:boosted_sketch_oja}, which uses a shared Gaussian matrix across all sample blocks, making the population center common to the independent runs. This lets us apply geometric aggregation to boost the guarantee \eqref{eq:random_sketch_near_center}. We make two further observations: first, although it is written with a specified horizon $n$, Algorithm~\ref{alg:boosted_sketch_oja} is an online algorithm, as the iterations do not use knowledge of $n$. Second, it is implementable using $O(dR^2) = O(d\log^2(\frac 1 \zeta))$ space, for the eventual $R = O(\log(\frac 1 \zeta))$ in Theorem~\ref{thm:high_probability_oja_absolute}, by storing only the $\my_r$.

\begin{algorithm}[t]
\DontPrintSemicolon
\caption{$\mathsf{BoostedSketchOja}(\{\ma_t, \eta_t\}_{t \in [n]},R, \tau)$}
\label{alg:boosted_sketch_oja}
\codeInput $\{\ma_t \in \R^{d \times d}, \eta_t > 0\}_{t \in [n]}$, $R \in \N$, $\tau > 0$ \;
$\mg\gets $ $d \times R$ matrix with i.i.d.\ $\Nor(0,1)$ entries \;
\For{$r\in[R]$}{
    $\my_r\gets\mg$ \;
}
\For{$j \in [\lfloor \frac n R\rfloor]$}{
\For{$r \in [R]$}{
$\my_r \gets(\mi_d+\eta_{(j - 1)R + r}\ma_{(j-1)R+r})\my_r$\;
}
}
\For{$r\in[R]$}{
$\mz_{r}\gets \my_r/\normf{\my_r}$, or a fixed unit-Frobenius-norm matrix if $\my_r=\0_{d\times R}$
}
\For{$r\in[R]$}{
    $\calN_{r}\gets\{r'\in[R]:\normf{\mz_{r}-\mz_{r'}}\leq2\tau\}$ \;
}
$i \gets $ any index with $|\calN_i| \ge \frac R 2$, else $i \gets 1$\;\label{line:select_center}
$\vw_n \gets \textup{any top left singular vector of } \mz_{i}$\;
\codeReturn $\vw_n$ 
\end{algorithm}

We first give the guarantee for a common stepsize schedule across the $R$ runs.

\begin{lemma}\label{lem:oja_aggregation}
Under Model~\ref{model:oja_special}, let $R,T\in\N$ and run $R$ independent Oja products for $T$ steps with common step sizes $\eta_t$, using the shared Gaussian initialization, $\mg$, and aggregation rule of Algorithm~\ref{alg:boosted_sketch_oja}. Let $\gam\in(0,1)$ and $0<\tau<\frac{1}{4\sqrt R}$. If $0<\eta_t\le\frac1{\lam_1}$ and
\[
V\sum_{t=1}^T\eta_t^2\le\frac{\tau^2}{128},
\qquad
4d\exp\Par{-\gam\lam_1\sum_{t=1}^T\eta_t}\le\tau^2,
\]
then the aggregated output is a $(\gam,16R\tau^2)$-cPCA of $\msig$ with probability $\ge1-4\exp(-\frac R{72})$.
\end{lemma}
\begin{proof}
Set $n=RT$ and index the update in round $j$ of run $r$ by $\ma_{(j-1)R+r}$, for $j\in[T]$ and $r\in[R]$. Let $\mpp$ project onto eigenvectors of $\msig$ with eigenvalues below $(1-\gam)\lam_1$. Define
\begin{align*}
\mc_j &\defeq (\mi_d+\eta_j\msig)\cdots(\mi_d+\eta_1\msig),\\
      \mb_j^{(r)}
    &\defeq
    (\id_d+\eta_j\ma_{(j-1)R+r})\cdots
    (\id_d+\eta_1\ma_{r}),\,
\end{align*}
and \[
\mz_\star \defeq \frac{\mc_T\mg}{\normf{\mc_T\mg}}, \qquad \mz_r \defeq \frac{\mb_T^{(r)}\mg}{\normf{\mb_T^{(r)}\mg}},
\]
with the algorithm's fallback definition when $\mb_T^{(r)}\mg=\0_{d\times R}$.
By Lemma~\ref{lem:product_sketch_concentration} with $t\gets T$ and $q_T\gets V\sum_{j\in[T]}\eta_j^2\le\frac{\tau^2}{128}$, with probability at least $1-3\exp(-\frac R {16})$ over $\mg$, $\normf{\mpp\mz_\star}^2\le \tau^2$. Further, for each run,
\[\Pr(\normf{\mz_\star - \mz_r}\le \tau \mid \mg)\ge \frac 34.\]
Conditioning on this event for $\mg$, the runs are independent. By a Chernoff bound, the set \[
\calG\defeq \Brace{r\in[R]\mid \normf{\mz_r - \mz_\star} \le \tau}
\]
has size at least $\frac {2R} 3$ with conditional probability $\ge1-\exp(-\frac R{72})$. Every such candidate has at least $\frac {2R} 3$ neighbors within distance $2\tau$. Conversely, any candidate with at least $\frac R 2$ such neighbors has a neighbor in $\calG$, and is therefore within distance $3\tau$ of $\mz_{\star}$ by the triangle inequality.

Let $\widehat{\mz}\defeq \mz_i$ be the matrix selected in Line~\ref{line:select_center}. On the preceding events, which hold with probability $\ge1-3\exp(-\frac R{16})-\exp(-\frac R{72})\ge1-4\exp(-\frac R{72})$,
\[
    \normf{\mpp\widehat{\mz}}
    \leq \normf{\mpp \mz_{\star}} + \normf{\widehat{\mz}-\mz_{\star}}\le 4\tau.
\]
Let $\vw_n$ be a top left singular vector of $\widehat{\mz}$. Since $\normsf{\widehat{\mz}}=1$ and $\rank(\widehat\mz)\leq R$, its top singular value is at least $R^{-1/2}$, so $\normsf{\mpp \widehat \mz}^2 \ge \frac 1 R \norms{\mpp \vw_n}_2^2$ by expanding the singular value decomposition. Thus,
\[
    \norm{\mpp\vw_n}_2^2
    \leq
    R\normf{\mpp\widehat{\mz}}^2
    \leq
    16R\tau^2.
\]
\end{proof}

We conclude by proving our main high-probability result, Theorem~\ref{thm:high_probability_oja_absolute}.


\begin{theorem}
\label{thm:high_probability_oja_absolute}
Let $\zeta\in(0,\frac13)$ and $(\gam,\Delta)\in(0,1)^2$. Under Model~\ref{model:oja_special}, if
\[n = \Omega\Par{\frac{V}{\lam_1^2\gamma^2\Delta} \log^2\Par{\frac d {\Delta\zeta}}\log^2\Par{\frac 1 \zeta} + \frac{\log\Par{\frac d {\Delta\zeta}}\log\Par{\frac 1 \zeta}}{\gamma}}\]
for an appropriate constant, then there exist choices of the inputs $R, \tau, \{\eta_t\}_{t \ge 1}$, such that the output of Algorithm~\ref{alg:boosted_sketch_oja} is a $(\gamma,\Delta)$-cPCA of $\msig$ with probability $\ge 1 - \zeta$.
\end{theorem}

\begin{proof}
Throughout the proof, for sufficiently large $C > 0$ and small $c > 0$, we take
\[R \defeq \left\lceil C\log\Par{\frac 1 \zeta}\right\rceil,\quad L \defeq C\log\Par{\frac d {\Delta\zeta}},\quad \tau \defeq c\sqrt{\frac \Delta R}.\]
For an iteration $t \in [n]$ such that $t = (j - 1)R + r$ for some $r \in [R]$, we also choose
\[\eta_t \defeq \frac{L}{\gamma\lam_1(\beta + j)},\text{ where } \beta \defeq C\Par{\frac{L}{\gamma} + \frac{VRL^2}{\gamma^2\lam_1^2\Delta}}. \]

Let $T\defeq \lfloor \frac nR\rfloor$ and write $\alpha_j\defeq\eta_{(j-1)R+r}$ for the common stepsize in round $j\in[T]$, which is independent of $r\in[R]$.
By the bound on $\beta$ and $n$, we ensure $T\ge 4\beta$ and  $\alpha_j\le \frac{1}{\lambda_1}$. Then, the same integral comparison as in the proof of Theorem~\ref{thm:gapfree_oja} yields \[
q_T = V\sum_{j\in[T]}\alpha_j^2\le \frac{VL^2}{\gam^2\lambda_1^2\beta}\le \frac{1}{128}, \quad \gam\lambda_1\sum_{j\in[T]}\alpha_j \ge L,
\]
if $C$ is sufficiently large relative to $c$. These imply
\[
4d\exp\Par{-\gam\lambda_1\sum_{j\in[T]}\alpha_j}\le 4d e^{-L}\le \tau^2, \quad 8\sqrt{2q_T}\le c\sqrt{\frac \Delta R}=\tau.\]
Applying Lemma~\ref{lem:oja_aggregation} to these $R$ runs of $T$ updates, with $\eta_j\gets\alpha_j$ and radius $\tau$, gives
\[
\norm{\mpp\vw_n}_2^2\le16R\tau^2=16c^2\Delta\le\Delta,
\]
for $c\le\frac14$, with failure probability at most $4\exp(-\frac R{72})\le\zeta$ when $C$ is sufficiently large.
\end{proof}
\section{Energy PCA Guarantees}
\label{sec:epca}

In this section, we show how to extend the approach of Section~\ref{sec:oja} to an alternative gap-free notion of PCA often considered in the literature, energy PCA, which asks for a direction capturing nearly the largest possible variance. We use the following definition from \cite{JambulapatiKLPPT24}.

\begin{definition}[ePCA]\label{def:epca}
Let $\alpha\in(0,1)$, and let $\msig\in\PSD^{d\times d}$. We say that a unit vector $\vw \in \R^d$ is an \emph{$\alpha$-ePCA} (energy PCA) of $\msig$ if $\vw^\top\msig\vw\ge(1-\alpha)\lam_1(\msig)$.
\end{definition}

Lemma~8 of \cite{JambulapatiKLPPT24} shows that a $(\gam,\Delta)$-cPCA is also a $(\gam+\Delta)$-ePCA. Taking $\gam=\Delta\gets\frac\alpha2$ in Theorem~\ref{thm:gapfree_oja} therefore gives an $\alpha$-ePCA, but the resulting sample complexity scales as $\frac{V}{\lam_1^2\alpha^3}$, which is suboptimal in its dependence on $\frac1\alpha$. This conversion bounds the squared projection onto eigenvalues below a single threshold $(1-\gam)\lam_1$. To avoid the lossy conversion, we will use the following identity that expresses $1-\frac{\vw^\top\msig\vw}{\lam_1}$ as an integral over thresholds $(1-u)\lam_1$.

\begin{proposition}[Multiscale cPCA-to-ePCA]
\label{prop:multiscale_cpca_epca}
Let $\msig\in\PSD^{d\times d}$ with $\lam_1\defeq\lam_1(\msig)>0$, and let $\vw\in\R^d$ be a unit vector. For every $u\in(0,1)$, let $\mpp_u\in\PSD^{d\times d}$ project onto the eigenvectors of $\msig$ with eigenvalues $<(1-u)\lam_1$, and define $\Delta_u\defeq\norm{\mpp_u\vw}_2^2$. Then
\begin{equation}\label{eq:epca_spectral_integral}
\id_d-\frac{\msig}{\lam_1}=\int_0^1\mpp_u\dd u.
\end{equation}
In particular,
\begin{equation}\label{eq:epca_profile_identity}
1-\frac{\vw^\top\msig\vw}{\lam_1}
=\int_0^1\Delta_u\dd u.
\end{equation}
Consequently, if $\alpha\in(0,1)$ and $\int_0^1\Delta_u\dd u\le\alpha$, then $\vw$ is an $\alpha$-ePCA of $\msig$.
\end{proposition}

\begin{proof}
Write $\msig=\sum_{j \in [d]}\lam_j\vu_j\vu_j^\top$. Then since
\[
\mpp_u=\sum_{j\in[d]}
\ind_{\Brace{u<1-\frac{\lam_j}{\lam_1}}}\vu_j\vu_j^\top,
\]
integrating each indicator gives the matrix identity in \eqref{eq:epca_spectral_integral}.
Taking the quadratic form with $\vw$ proves the integral identity \eqref{eq:epca_profile_identity} and the claim.
\end{proof}

We apply Proposition~\ref{prop:multiscale_cpca_epca} to the output $\vw_n$ of Algorithm~\ref{alg:oja}, so henceforth $\Delta_u\defeq\norm{\mpp_u\vw_n}_2^2$. Suppose $\eta_t\le\frac1{\lam_1}$ for all $t\in[n]$ and $V\sum_{t\in[n]}\eta_t^2$ is a sufficiently small constant. For each fixed $u\in(0,1)$, combining Lemmas~\ref{lem:gapfree_one_step_power} and~\ref{lem:gapfree_bad_trace_vary_step} with $\gam\gets u$ and fixed constant $\zeta$ gives, for a universal constant $C$,
\[
\Delta_u\le C\min\Brace{1,
d\exp\Par{-u\lam_1\sum_{t\in[n]}\eta_t}
+V\sum_{t\in[n]}\eta_t^2}.
\]
These bounds hold with constant probability for each fixed $u$, but need not hold simultaneously for all $u$. Even if the bound held simultaneously for every $u$, integrating it would only give
\begin{equation}\label{eq:naive_integral}
\int_0^1\Delta_u\dd u
=O\Par{
\frac{\log(ed)}{\lam_1\sum_{t\in[n]}\eta_t}
+V\sum_{t\in[n]}\eta_t^2}.
\end{equation}
For step sizes $\eta_t=\frac{L}{n+t}$, making both terms in the integral bound \eqref{eq:naive_integral} at most $\alpha$ requires
\[
L=\Omega\Par{\frac{\log(ed)}{\lam_1\alpha}},
\quad
n=\Omega\Par{\frac{VL^2}{\alpha}}
=\Omega\Par{\frac{V\log^2(ed)}{\lam_1^2\alpha^3}},
\]
which is still suboptimal. The following lemma gives a sharper bound on $\int_0^1\Delta_u\dd u$.

\begin{restatable}{lemma}{restateintegratedepca}\label{lem:epca_integrated_profile}
Under Model~\ref{model:oja_special} with the products $\mb_t,\mc_t$ defined in \eqref{eq:bi_def}, suppose $0<\eta_s\le\frac1{\lam_1}$ for all $s\in[n]$. For $u\in(0,1)$, define $\Delta_u\defeq\norm{\mpp_u\vw_n}_2^2$ for the projector $\mpp_u$ in Proposition~\ref{prop:multiscale_cpca_epca}.
There are universal constants $c,C>0$ such that, if $V\sum_{s\in[n]}\eta_s^2\le c$, then, with probability $\ge\frac{3}{4}$,
\begin{equation}\label{eq:epca_integrated_profile}
\int_0^1\Delta_u\dd u
\le C\Par{
\frac{\log(ed)}{\lam_1\sum_{s\in[n]}\eta_s}
+V\sum_{i\in[n]}\eta_i^2\min\Brace{1,\frac1{\lam_1\sum_{s=i+1}^n\eta_s}}
}.
\end{equation}
Here the minimum for $i=n$ is interpreted as $1$.
\end{restatable}
The proof retains the dependence on $u$ in the bound on $\E\normf{\mpp_u\mb_n}^2$. After division by $\normf{\mc_n}^2$, this bound contains, up to a universal constant,
\[
V\sum_{i\in[n]}\eta_i^2\exp\Par{-\frac{u\lam_1}{2}\sum_{s=i+1}^n\eta_s}.
\]
Integrating each summand over $u$ gives the factor $\min\{1,1/(\lam_1\sum_{s=i+1}^n\eta_s)\}$ in \eqref{eq:epca_integrated_profile}, up to a universal constant.
In Appendix~\ref{sec:epca_proofs}, we provide a proof of Lemma~\ref{lem:epca_integrated_profile} where we integrate the bounds on $\E\normf{\mpp_u\mb_n}^2$ over $u$ before applying Markov's inequality and controlling the normalization of $\vw_n$. This avoids a simultaneous union bound over all thresholds $u\in(0,1)$.

We next evaluate the error bound for the schedule used in both this section and Section~\ref{ssec:private_epca}.

\begin{lemma}\label{lem:epca_schedule}
Under Model~\ref{model:oja_special}, let $a\ge1$, $n\ge4a$, and $\eta_t=\frac{a}{\lam_1(n/4+t)}$ for $t\in[n]$. There are universal constants $c,C>0$ such that, if $V\sum_{t\in[n]}\eta_t^2\le c$, then with probability at least $\frac34$,
\begin{equation}\label{eq:epca_schedule_error}
1-\frac{\vw_n^\top\msig\vw_n}{\lam_1}
\le C\Par{\frac{\log(ed)}a+\frac{Va\log(ea)}{\lam_1^2n}}.
\end{equation}
\end{lemma}
\begin{proof}
Our strategy is to bound the two terms in Lemma~\ref{lem:epca_integrated_profile}
for this schedule, then apply Proposition~\ref{prop:multiscale_cpca_epca} to obtain
the ePCA bound. Since $\eta_t\lam_1\le \frac{4a} n\le1$, Lemma~\ref{lem:epca_integrated_profile} applies. The schedule satisfies
\[
\lam_1\sum_{t\in[n]}\eta_t\ge\frac{4a}{5},\quad
\lam_1\sum_{s=i+1}^n\eta_s\ge\frac{4a(n-i)}{5n},\quad
\eta_i^2\le\frac{16a^2}{\lam_1^2n^2}.
\]
The bound on $\lam_1\sum_{t\in[n]}\eta_t$ controls the first term in Lemma~\ref{lem:epca_integrated_profile}
by $O(\frac{\log(ed)}{a})$. To bound the remaining variance sum, we write $k=n-i$ and
use the bounds on $\lam_1\sum_{s=i+1}^n\eta_s$ and $\eta_i^2$:
\[
\sum_{i\in[n]}\eta_i^2\min\Brace{1,\frac1{\lam_1\sum_{s=i+1}^n\eta_s}}
\le\frac{16a^2}{\lam_1^2n^2}\Par{1+\sum_{k\in[n-1]}\min\Brace{1,\frac{5n}{4ak}}}
=O\Par{\frac{a\log(ea)}{\lam_1^2n}}.
\]
Finally, applying Lemma~\ref{lem:epca_integrated_profile} with our choices of $\eta_t$, and Proposition~\ref{prop:multiscale_cpca_epca} with $\vw\gets\vw_n$, gives with probability at least $\frac34$,
\[
1-\frac{\vw_n^\top\msig\vw_n}{\lam_1}
=\int_0^1\Delta_u\dd u
\le C\Par{\frac{\log(ed)}a+\frac{Va\log(ea)}{\lam_1^2n}}.
\]
\end{proof}

We therefore obtain the following ePCA guarantee.

\begin{theorem}\label{thm:gapfree_epca}
Under Model~\ref{model:oja_special}, let $\alpha\in(0,\frac{1}{4})$. Then if
\[
n = \Omega\Par{
\frac{V\log^2\Par{\frac d\alpha}}{\lam_1^2\alpha^2}
+\frac{\log\Par{\frac d\alpha}}{\alpha}
}
\]
for an appropriate constant,
then there exists $a\in \R_{>0}$ such that taking $\eta_t=\frac{a}{\lam_1(n/4+t)}$ for $t \in [n]$, the output of Algorithm~\ref{alg:oja} is an $\alpha$-ePCA of $\msig$ with probability $\ge\frac34$.
\end{theorem}


\begin{proof}
Let $L\defeq\log(\frac d \alpha)$ and take $a=\frac{C_1L} \alpha$ for a sufficiently large universal constant $C_1$. Then our assumptions ensure $n\ge4a$ and
\[
V\sum_{t\in[n]}\eta_t^2\le\frac{4Va^2}{\lam_1^2n}\le c,
\]
where $c$ is the constant in Lemma~\ref{lem:epca_schedule}. Applying Lemma~\ref{lem:epca_schedule} with $a\gets \frac{C_1L}\alpha$ and using $\log(ea)=O(L)$ gives, with probability at least $\frac34$,
\[
1-\frac{\vw_n^\top\msig\vw_n}{\lam_1}
=O\Par{\frac La+\frac{VaL}{\lam_1^2n}}
=O\Par{\frac\alpha{C_1}+\frac{C_1\alpha}{C}}.
\]
Choosing $C_1$ and then $C$ sufficiently large makes the error at most $\alpha$.
\end{proof}

Although Theorem~\ref{thm:gapfree_epca} is stated with a constant failure probability, it is straightforward to use holdout samples to reduce its failure probability.

\begin{corollary}\label{cor:epca_whp}
Let $(\zeta, \alpha) \in (0, \frac 1 4)^2$. Under Model~\ref{model:oja_special}, if
\[n = \Omega\Par{\frac{V\log^2\Par{\frac d\alpha}\log(\frac 1 \zeta)}{\lam_1^2\alpha^2}
+\frac{\log\Par{\frac d\alpha}\log(\frac 1 \zeta)}{\alpha}}\]
for an appropriate constant, we can obtain an $\alpha$-ePCA of $\msig$ with probability $\ge 1 - \zeta$.
\end{corollary}
\begin{proof}
Fix some unit vector $\vw \in \R^d$. We claim that we can estimate its quadratic form $\vw^\top \msig \vw$ up to additive error $\frac{\alpha\lam_1}{4}$ using $O(\frac{V}{\alpha^2 \lam_1^2})$ holdout samples. To see this, under Model~\ref{model:oja_special}, a single sample quadratic form $\vw^\top \ma_t \vw$ is unbiased for $\vw^\top \msig \vw$, and has variance at most $V$:
\[\E\Brack{\Par{\vw^\top\Par{\ma_t - \msig}\vw}^2} \le \E\norm{\Par{\ma_t - \msig}\vw}_2^2 \le V.\]
Thus, averaging independent estimates and applying Chebyshev's inequality gives the claim.

Now, calling Theorem~\ref{thm:gapfree_epca} $O(\log \frac 1 \zeta)$ times independently with $\alpha \gets \frac \alpha 2$ implies that with probability $\ge 1 - \frac \zeta 2$, at least one of the outputs will be an $\frac \alpha 2$-ePCA. Taking the median of $O(\log \frac 1 \zeta)$ estimates of the quadratic form $\vw^\top \msig \vw$ obtained by each output $\vw$, using independent holdout samples, then yields the unit vector with largest quadratic form up to additive error $\frac {\alpha \lam_1} 2$, concluding the proof. Note that the same holdout samples are simultaneously accurate for each output by applying independence and taking a union bound, and do not dominate the stated sample complexity.
\end{proof}

\textbf{Matching lower bound.}
\label{ssec:epca_lower_bound}
We briefly conclude the section by showing a matching lower bound, up to logarithmic factors, by appealing to Theorem~\ref{thm:pca_lower}.
\begin{corollary}\label{cor:epca_lower_bound}
    Fix $d\ge2$ and any choice of $\lam_1 > 0$, $V > 0$, $\alpha\in(0,\frac1{32})$. There is no algorithm $\calA$ that takes as input $\{\ma_i\}_{i \in [n]}$ from Model~\ref{model:oja_special}, and outputs an $\alpha$-ePCA of $\msig$ with probability $\ge \frac 2 3$, even assuming that $\ma_i \in \PSD^{d \times d}$ for all $i \in [n]$, unless for an appropriate constant,
\[n = \Omega\Par{\frac{V}{\lam_1^2 \alpha^2}}.\]
\end{corollary}
\begin{proof}
    By Lemma 7, \cite{JambulapatiKLPPT24} with $k=1$, an $\alpha$-ePCA of $\msig$ is also a $(\gam,\frac\alpha\gam)$-cPCA for every $\gam\in(\alpha,1)$, in the notation of Definition~\ref{def:cpca}. Then, taking $\gam\gets8\alpha$ and $\Delta\gets\frac18$ in Theorem~\ref{thm:pca_lower} gives
\[
n=\Omega\Par{\frac{V}{\lam_1^2(8\alpha)^2\cdot \frac 1 8}}
=\Omega\Par{\frac{V}{\lam_1^2\alpha^2}}.
\]
\end{proof}

\section{Application to Differentially Private PCA}
\label{sec:dp_pca}

In this section, we give our application to differentially private PCA. Given i.i.d.\ sub-Gaussian samples with covariance $\msig$, we seek a cPCA or ePCA subject to the following privacy guarantee.

\begin{definition}[Differential privacy]\label{def:dp}
We say that a randomized algorithm $\alg: (\R^d)^n \to \Omega$ is $(\eps, \delta)$-differentially private if for all measurable $\calE \subseteq \Omega$, and all $S, S' \in (\R^d)^n$ differing in one entry,
\[\Pr[\alg(S)\in \calE]\le \exp(\eps)\Pr[\alg(S')\in \calE]+\delta.\]
\end{definition}

Our utility analysis holds under the following assumption on the dataset.

\begin{definition}[$\nu$-sub-Gaussianity]
\label{definition:hypercontractivity}
A distribution $\calD$ on $\R^d$ is $\nu$-sub-Gaussian if, for every $\vv\in\R^d$,
\[\E_{\vx\sim\calD}\Brack{\exp\Par{\inner{\vx-\E\vx}{\vv}}}
\le \exp\Par{\frac{\nu^2\norm{\vv}_2^2}{2}}.\]
\end{definition}

\begin{model}
\label{model:private_gapfree_oja}
The samples $\vx_1,\ldots,\vx_n$ are drawn i.i.d.\ from a mean-zero, $\nu$-sub-Gaussian distribution with covariance $\msig\in\PSD^{d\times d}$ and $\lam_1=\normop{\msig}>0$.
\end{model}

Throughout this section, we treat $n,\nu,\lam_1$ as public parameters, fixed independently of the dataset. As is standard for statistical DP algorithms, our privacy guarantee will hold regardless of the input dataset, and our utility guarantee will hold assuming that the dataset follows Model~\ref{model:private_gapfree_oja}.

We next state our main algorithm in Algorithm~\ref{alg:private_gapfree_batched_oja}, which is patterned off of the DP-PCA algorithm of \cite{LiuKJO22}. After clipping the input dataset, the algorithm simply adds an appropriate Gaussian perturbation to each iterate of an empirical power method, which we show can be cast as an instance of Model~\ref{model:oja_special}. One major difference between Algorithm~\ref{alg:private_gapfree_batched_oja} and the variant in \cite{LiuKJO22} is that we do not subdivide our dataset into minibatches, and instead use full-batch iterations; this difference ends up shaving a roughly $\gamma^{-1/2}$ factor from our final sample complexity.

\begin{algorithm}[!hbt]
\DontPrintSemicolon
\caption{$\mathsf{PrivBoostedSketchOja}(\{\vx_i\}_{i\in[n]},\{\eta_t\}_{t\in[T]},R,T,R_1,R_2,\tau,\sigma)$}
\label{alg:private_gapfree_batched_oja}
\codeInput $\{\vx_i\in\R^d\}_{i\in[n]}$, $\{\eta_t>0\}_{t\in[T]}$, $(R,T)\in\N^2$, $(R_1,R_2,\tau,\sigma)\in\R_{>0}^4$\;
$\mg\gets$ $d\times R$ matrix with i.i.d.\ $\Nor(0,1)$ entries\;
\For{$r\in[R]$}{$\my_r\gets\mg$\;}
\For{$t\in[T]$}{\label{line:for_start}
  \For{$r\in[R]$}{
    $\mmu_r\md_r\mv_r^\top\gets$ compact SVD of $\my_r$, with $s_r\defeq\rank(\my_r)$\;\label{line:svd}
    $\mq_{r,t}\gets\0_{d\times s_r}$\;
    \For{$i\in[n]$}{
      \tcp{Use multiplier $1$ if a denominator in either clipping step is zero.}
      $\vs_i\gets\vx_i\min\{1,\sqrt{R_1}/\norm{\vx_i}_2\}$\;\label{line:clip_1}
      $\vz_i\gets\vs_i\min\{1,\sqrt{R_2}/\norm{\mmu_r^\top\vs_i}_2\}$\;\label{line:clip_2}
      $\mq_{r,t}\gets\mq_{r,t}+\frac1n\vz_i(\vz_i^\top\mmu_r)$\;
    }
    $\mh_{r,t}\gets$ $d\times s_r$ matrix with i.i.d.\ $\Nor(0,\sigma^2)$ entries\;\label{line:sigma}
    $\my_r\gets\my_r+\eta_t(\mq_{r,t}+\mh_{r,t})\md_r\mv_r^\top$\;
  }
}\label{line:for_end}
\For{$r\in[R]$}{
  $\mz_r\gets\my_r/\normf{\my_r}$, or a fixed unit-Frobenius-norm matrix if $\my_r=\0_{d\times R}$\;
}
\For{$r\in[R]$}{
  $\calN_r\gets\{r'\in[R]:\normf{\mz_r-\mz_{r'}}\le2\tau\}$\;
}
$i\gets$ any index with $|\calN_i|\ge R/2$, else $i\gets1$\;
$\vw_n\gets$ any top left singular vector of $\mz_i$\;
\codeReturn $\vw_n$
\end{algorithm}

\subsection{Privacy}

We next prove that Algorithm~\ref{alg:private_gapfree_batched_oja} satisfies $(\eps, \delta)$-DP when $\sigma$ in Line~\ref{line:sigma} is appropriately chosen. Our proof is standard, and proceeds via \emph{R\'enyi DP}, an alternative privacy accounting strategy that is particularly well-suited to the \emph{Gaussian mechanism}. For brevity, we defer background on the Gaussian mechanism to Appendix A of \cite{DworkR14}, and background on R\'enyi DP to \cite{Mironov17}.

\begin{lemma}
\label{lem:dp_pca_privacy}
For $\eps \in (0, 1]$ and $\delta \in (0, \frac 1 3)$, the output of Algorithm~\ref{alg:private_gapfree_batched_oja} is $(\eps, \delta)$-DP if
\begin{equation}\label{eq:dp_noise_variance}
\sigma^2\ge\frac{12RTR_1R_2}{n^2\eps^2}\log\Par{\frac1\delta}.
\end{equation}
Moreover, every $\mmu_r$ on Line~\ref{line:svd} is $(\eps, \delta)$-DP for all $r \in [R]$, at every iteration $t \in [T]$.
\end{lemma}

\begin{proof}
The algorithm only accesses the dataset $\{\vx_i\}_{i \in [n]}$ in the nested for loops from Lines~\ref{line:for_start} to~\ref{line:for_end}. Fix the initialization and preceding noisy answers at the beginning of one loop, indexed by $t \in [T]$ and $r \in [R]$, so that $\my_r$ and $\mmu_r$ are fixed. Each clipped summand to $\mq_{r, t}$ satisfies
\[
\normf{\vz_i\vz_i^\top\mmu_r}
=\norm{\vz_i}_2\norm{\mmu_r^\top\vz_i}_2
\le\sqrt{R_1R_2},
\]
and $\vz_i$ is a deterministic function of $\muu_r$ and $\vx_i$,
so replacing one sample $\vx_i$ changes $\mq_{r, t}$ in Frobenius norm by $\le s\defeq\frac 2 n \sqrt{R_1R_2}$. Thus, Proposition 7 and Corollary 3 of \cite{Mironov17} show that the noisy answer $\mq_{r,t}+\mh_{r,t}$ is $(p, \frac{ps^2}{2\sigma^2})$-RDP, and $\my_r$ is updated by a deterministic function of this answer and the preceding state. Now RDP composition (Proposition 1, \cite{Mironov17}) over $RT$ iterations shows that the transcript of all $\my_r$ is $(p, \rho)$-RDP, where
\[\rho = \frac{ps^2RT}{2\sigma^2}.\]
Finally, taking $p = 1 + \frac{2\log (\frac 1 \delta)}{\eps}$, and using the lower bound on $\sigma^2$ in \eqref{eq:dp_noise_variance}, implies
\[\frac{\log \Par{\frac 1 \delta}}{p - 1} \le \frac \eps 2,\quad \rho \le \frac{ps^2 n^2\eps^2}{24R_1R_2\log \Par{\frac 1 \delta}} = \frac{p\eps^2}{6\log\Par{ \frac 1 \delta}} \le \frac{\eps^2}{6\log\Par{ \frac 1 \delta}} + \frac{\eps}{3} \le \frac \eps 2.\]

Proposition 3 of \cite{Mironov17} then shows that the transcript of all of the $\my_r$ is $(\eps, \delta)$-DP.
The privacy of the algorithm's output and all $\muu_r$ then follows, as postprocessings of the transcript.
\end{proof}

\subsection{Utility}

We next give our utility analysis. For convenience, denote the dataset and empirical covariance by
\[S\defeq\{\vx_i\}_{i\in[n]},\quad \hmsig\defeq \frac 1 n\sum_{i\in[n]}\vx_i\vx_i^\top.\]
We first show how to couple iterates of Algorithm~\ref{alg:private_gapfree_batched_oja} to an instance of Model~\ref{model:oja_special}. To begin, we show that with high probability, the clipping events on Lines~\ref{line:clip_1} and~\ref{line:clip_2} never occur. This step requires using our earlier privacy guarantee to handle a dependence between $\muu$ and the dataset $S$.

\begin{lemma}
\label{lem:dp_adaptive_projection}
Under Model~\ref{model:private_gapfree_oja}, let $\alg$ be an $(\eps,\delta)$-DP mechanism whose output $\mmu=\alg(S)$ is an orthonormal matrix with at most $R$ columns. For every $i\in[n]$ and $u>0$,
\[
\Pr\Par{\norm{\mmu^\top\vx_i}_2^2>3\nu^2(R+u)}
\le \exp({\eps-u})+\delta.
\]
\end{lemma}

\begin{proof}
Replace $\vx_i$ by an independent copy to form $S^{(i)}$, and set $\mmu'=\alg(S^{(i)})$. Conditional on $\mmu'$, the vector $\vx_i$ remains mean-zero and $\nu$-sub-Gaussian. Theorem~2.1 of \cite{HsuKZ12}, applied to $(\mmu')^\top\vx_i$, gives $\Pr(\norm{(\mmu')^\top\vx_i}_2^2>3\nu^2(R+u))\le\exp(-u)$. Applying Definition~\ref{def:dp} to the neighboring datasets $S,S^{(i)}$ and averaging over the independent copy gives the claim.
\end{proof}

For a failure probability parameter $\zeta\in(0,1)$, set
\begin{equation}
\label{eq:dp_pca_R1_R2}
R_1\defeq3\nu^2\Par{d+\log\Par{\frac{8n}{\zeta}}},\qquad
R_2\defeq3\nu^2\Par{R+2\log\Par{\frac{16nRT}{\zeta}}}.
\end{equation}
We use the smaller privacy failure parameter
\begin{equation}\label{eq:dp_coupling_delta}
\delta_*\defeq\min\Brace{\delta,\frac{\zeta}{16nRT}},
\end{equation}
so that the additive privacy errors can be summed over all sample projections.

\begin{lemma}
\label{lem:dp_update_coupling}
Under Model~\ref{model:private_gapfree_oja}, let $\eps\in(0,1]$, $\delta\in(0,\frac13)$, and $\zeta\in(0,1)$. Choose $R_1,R_2$ as in \eqref{eq:dp_pca_R1_R2}, and suppose $\sigma^2$ satisfies the lower bound in \eqref{eq:dp_noise_variance} with $\delta\gets\delta_*$ from \eqref{eq:dp_coupling_delta}. There exist matrices
\[
\ma_{(t-1)R+r}\defeq\hmsig+\mg_{r,t},
\]
where the $\mg_{r,t}\in\R^{d\times d}$ have independent $\Nor(0,\sigma^2)$ entries, such that the updates in Algorithm~\ref{alg:private_gapfree_batched_oja} can be coupled to $\my_r\gets(\id_d+\eta_t\ma_{(t-1)R+r})\my_r$ with probability $\ge1-\frac\zeta4$. Conditional on the dataset, these matrices are i.i.d.\ with mean $\hmsig$ and satisfy both variance bounds in Model~\ref{model:oja_special} with $V=d\sigma^2$.
\end{lemma}

\begin{proof}
We first show that neither clipping step changes any sample, except with probability $\frac\zeta4$. We then choose the Gaussian matrices so that the private and Oja updates agree whenever no clipping occurs. By Theorem~2.1 of \cite{HsuKZ12} and the choice of $R_1$,
\[
\Pr\Par{\exists i\in[n]:\norm{\vx_i}_2^2>R_1}
\le n\exp\Par{-\log\Par{\frac{8n}{\zeta}}}=\frac\zeta8.
\]
By Lemma~\ref{lem:dp_pca_privacy} with $\delta\gets\delta_*$, each subspace $\mmu_r$ computed by the algorithm is $(\eps,\delta_*)$-DP. For Line~\ref{line:clip_2}, apply Lemma~\ref{lem:dp_adaptive_projection} with $u=2\log(\frac{16nRT}\zeta)$, so that $R_2=3\nu^2(R+u)$. Since $\eps\le1\le\log(\frac{16nRT}\zeta)$, for every $(i,r,t)\in[n]\times[R]\times[T]$,
\[
\Pr\Par{\norm{\mmu_r^\top\vx_i}_2^2>R_2}
\le\exp(\eps-u)+\delta_*
\le\frac{\zeta}{16nRT}+\frac{\zeta}{16nRT}
=\frac{\zeta}{8nRT}.
\]
These bounds apply to the subspaces computed by the algorithm, including when earlier samples were clipped. A union bound gives total failure probability at most $\frac\zeta8+nRT\cdot\frac\zeta{8nRT}=\frac\zeta4$. On the complementary event, $\vz_i=\vx_i$ for every sample in every update.

Now draw the $d\times d$ matrices $\mg_{r,t}$ with independent $\Nor(0,\sigma^2)$ entries, independently across updates and independently of the data and initialization. Set $\mh_{r,t}=\mg_{r,t}\mmu_r$. Given the dataset and all preceding updates, $\mmu_r$ is fixed and orthonormal, so $\mh_{r,t}$ has independent $\Nor(0,\sigma^2)$ entries, as required by the algorithm. Whenever no clipping occurs, $\mq_{r,t}=\hmsig\mmu_r$, and the SVD identity gives
\[
\my_r+\eta_t(\hmsig\mmu_r+\mg_{r,t}\mmu_r)\md_r\mv_r^\top
=(\id_d+\eta_t(\hmsig+\mg_{r,t}))\my_r.
\]
Thus, starting from the same initialization, the private and Oja iterates agree with probability at least $1-\frac\zeta4$. Conditional on the dataset, the matrices $\ma_{(t-1)R+r}$ are i.i.d.\ with mean $\hmsig$, and
\[
\E[\mg_{r,t}\mg_{r,t}^\top]
=\E[\mg_{r,t}^\top\mg_{r,t}]
=d\sigma^2\id_d.
\]
This verifies Model~\ref{model:oja_special} with $V=d\sigma^2$, conditional only on the dataset.
\end{proof}

We now combine the coupling with the analysis of Algorithm~\ref{alg:boosted_sketch_oja} to obtain a private cPCA guarantee.

\begin{theorem}
\label{thm:dp_pca_utility}
Under Model~\ref{model:private_gapfree_oja}, let $\eps\in(0,1]$, $(\delta,\zeta)\in(0,\frac13)^2$, and $(\gamma,\Delta)\in(0,1)^2$. If
\begin{equation}\label{eq:dp_sample_complexity}
n= \Omega\Par{
\frac{\nu^4\Par{d+\log\Par{\frac1\zeta}}}{\lam_1^2\gamma^2\Delta}
+\frac{d\nu^2\log\Par{\frac1\zeta}\log\Par{\frac{2d}{\Delta}\log\Par{\frac1\zeta}}\log\Par{\frac n\zeta}\sqrt{\log\Par{\frac n{\delta\zeta}}}}{\eps\lam_1\gamma\sqrt\Delta}
}
\end{equation}
for an appropriate constant,
there is a choice of inputs to Algorithm~\ref{alg:private_gapfree_batched_oja} that gives an $(\eps,\delta)$-DP algorithm returning a $(\gamma,\Delta)$-cPCA of $\msig$ with probability $\ge1-\zeta$.
\end{theorem}

\begin{proof}

We run Algorithm~\ref{alg:private_gapfree_batched_oja} with
\[
R=\left\lceil72\log\Par{\frac8\zeta}\right\rceil,\quad
L=\log\Par{\frac{256dR}{\Delta}},\quad
T=\left\lceil\frac{40L}{\gamma}\right\rceil,\quad \eta_t=\frac{8L}{\lam_1(10L+\gamma t)},\quad
\tau=\frac18\sqrt{\frac\Delta R}.
\]

Choose $R_1,R_2$ as in \eqref{eq:dp_pca_R1_R2}, $\delta_*$ as in \eqref{eq:dp_coupling_delta}, and set
\begin{equation}\label{eq:dp_cpca_sigma}
\sigma^2=\frac{12RTR_1R_2}{n^2\eps^2}\log\Par{\frac1{\delta_*}}.
\end{equation}
\textbf{Privacy.}
Lemma~\ref{lem:dp_pca_privacy} with $\delta\gets\delta_*$ gives $(\eps,\delta_*)$-DP, and hence, $(\eps,\delta)$-DP.

\textbf{Utility.} Our utility proof strategy
is to obtain a cPCA of $\hmsig$ and transfer it to $\msig$. For this, we first
bound $\normsop{\hmsig-\msig}$, then couple the private updates to Oja and apply
Lemma~\ref{lem:oja_aggregation}. We combine these guarantees on their common
success event.
Since $\lam_1\le\nu^2$, the first term in the sample complexity \eqref{eq:dp_sample_complexity} and Theorem~6.5 of \cite{Wainwright19} give, with probability $\ge1-\frac\zeta4$,
\begin{equation}\label{eq:dp_covariance_event}
\normop{\hmsig-\msig}\le\frac{\lam_1\gamma\sqrt\Delta}{4},\qquad
\frac34\lam_1\le\normop{\hmsig}\le\frac54\lam_1.
\end{equation}
Let $\calE_{\mathrm{cov}}$ denote the event that the covariance bounds in \eqref{eq:dp_covariance_event} hold. As in the proof of Proposition~\ref{prop:gapfree_offline}, it suffices on this event to obtain a $(\frac\gamma6,\frac\Delta4)$-cPCA of $\hmsig$.

Apply Lemma~\ref{lem:dp_update_coupling} with the chosen $R,T,R_1,R_2,\sigma$ and failure parameter $\zeta$. Let $\calE_{\mathrm{cpl}}$ denote the event that the two sequences of iterates agree throughout; then $\Pr(\calE_{\mathrm{cpl}}^c)\le\frac\zeta4$. We analyze these Oja updates conditional on a dataset satisfying the covariance bounds in \eqref{eq:dp_covariance_event}; they are independent with mean $\hmsig$ and variance $V=d\sigma^2$.

To obtain the required cPCA of $\hmsig$, we check the hypotheses of Lemma~\ref{lem:oja_aggregation} for the $R$ coupled runs of $T$ updates with $(\msig,\gam,V)\gets(\hmsig,\frac\gamma6,d\sigma^2)$, the chosen schedule $\{\eta_t\}_{t\in[T]}$, and radius $\tau$. Its step-size condition holds since $\eta_t\normop{\hmsig}\le1$. Moreover, our parameter choices give
\begin{equation}\label{eq:dp_full_variance}
4d\exp\Par{-\frac\gamma6\normop{\hmsig}\sum_{t\in[T]}\eta_t}\le\tau^2,
\quad
q_T\defeq d\sigma^2\sum_{t\in[T]}\eta_t^2\le\frac{\tau^2}{128}.
\end{equation}
The exponential bound in \eqref{eq:dp_full_variance} uses the choices of $T$ and $\eta_t$. The bound on $q_T$ follows by substituting $R_1,R_2$ from \eqref{eq:dp_pca_R1_R2} and $\sigma$ from \eqref{eq:dp_cpca_sigma}, and using the second term in the sample bound \eqref{eq:dp_sample_complexity}; the sample bound \eqref{eq:dp_sample_complexity} ensures $R,T=O(n)$ and hence $\log(\frac 1 {\delta_*})=O(\log\frac{n}{\delta\zeta})$.

Having verified the hypotheses of Lemma~\ref{lem:oja_aggregation}, we apply it for every fixed $S$ in $\calE_{\mathrm{cov}}$ to get
\[
\norm{\mpp\vw_{RT}}_2^2\le16R\tau^2=\frac\Delta4
\]
with failure probability at most $4\exp(-\frac R{72})\le\frac\zeta2$, where $\mpp$ projects onto eigenvectors of $\hmsig$ with eigenvalues below $(1-\frac\gamma6)\normsop{\hmsig}$. Thus, if $\calE_{\mathrm{Oja}}$ denotes the event that the coupled Oja output is a $(\frac\gamma6,\frac\Delta4)$-cPCA of $\hmsig$, then $\Pr(\calE_{\mathrm{Oja}}^c\mid S)\le\frac\zeta2$.

We have thus obtained the required cPCA of $\hmsig$. On $\calE_{\mathrm{cov}}\cap\calE_{\mathrm{cpl}}\cap\calE_{\mathrm{Oja}}$, the private output agrees with this Oja output, and the covariance bounds in \eqref{eq:dp_covariance_event}, via the proof of Proposition~\ref{prop:gapfree_offline}, make it a $(\gamma,\Delta)$-cPCA of $\msig$. The proof follows by noting that
\[
\Pr(\calE_{\mathrm{cov}}^c)+\Pr(\calE_{\mathrm{cpl}}^c)
+\Pr(\calE_{\mathrm{Oja}}^c\cap\calE_{\mathrm{cov}})
\le\frac\zeta4+\frac\zeta4+\frac\zeta2=\zeta.
\]
\end{proof}

\subsection{Private energy PCA}
\label{ssec:private_epca}

We conclude by giving an analogous private ePCA guarantee for Oja's algorithm.

\begin{theorem}\label{thm:private_epca}
Under Model~\ref{model:private_gapfree_oja}, let $\alpha\in(0,\frac14)$, $\eps\in(0,1]$, and $\delta\in(0,\frac13)$. If
\begin{equation}\label{eq:private_epca_samples}
n = \Omega\Par{
\frac{\nu^4d}{\lam_1^2\alpha^2}
+\frac{d\nu^2\log^2\Par{\frac{nd}\alpha}\sqrt{\log\Par{\frac{nd}{\alpha\delta}}}}{\eps\lam_1\alpha}
}
\end{equation}
for an appropriate constant,
there are choices of the inputs to Algorithm~\ref{alg:private_gapfree_batched_oja} that give an $(\eps,\delta)$-DP algorithm returning an $\alpha$-ePCA of $\msig$ with probability $\ge\frac23$.
\end{theorem}

\begin{proof} We run Algorithm~\ref{alg:private_gapfree_batched_oja} with $R=\tau=1$ and
\[
a\gets\frac{C_3\log\Par{\frac d\alpha}}{\alpha},\quad
T\gets \lceil C_4a\rceil,\quad
\eta_t\gets\frac{a}{\lam_1(T/4+t)},
\]
for sufficiently large universal constants $C_3,C_4$. We choose $R_1,R_2$ from \eqref{eq:dp_pca_R1_R2} with $\zeta=\frac1{12}$, set $\delta_*=\min\{\delta,\frac 1 {192nT}\}$ as in \eqref{eq:dp_coupling_delta}, and choose $\sigma$ as in \eqref{eq:dp_noise_variance} with $\delta\gets\delta_*$. The privacy proof is identical to Theorem~\ref{thm:dp_pca_utility}, so we focus on the utility proof.

First, Lemma~\ref{lem:dp_update_coupling} with $(R,\zeta)\gets(1,\frac1{12})$ gives coupled Oja updates with mean $\hmsig$ and variance $V=d\sigma^2$. Let $\calE_{\mathrm{cpl}}$ denote the event that Lemma~\ref{lem:dp_update_coupling} succeeds; then $\Pr(\calE_{\mathrm{cpl}}^c)\le\frac1{48}$.

Next, let $\calE_{\mathrm{cov}}$ denote the event $\normsop{\hmsig-\msig}\le\frac{\alpha\lam_1}8$. The first term in \eqref{eq:private_epca_samples} and Theorem~6.5 of \cite{Wainwright19} give $\Pr(\calE_{\mathrm{cov}}^c)\le\frac1{48}$. Fix a dataset $S$ in $\calE_{\mathrm{cov}}$ and analyze the Oja run conditional only on $S$.

To apply Lemma~\ref{lem:epca_schedule} to $\hmsig$, we express the step sizes using its top eigenvalue. Writing $\widehat\lam_1\defeq\normsop{\hmsig}$ and $a'\defeq \frac{a\widehat\lam_1}{\lam_1}$, the schedule becomes $\eta_t=\frac{a'}{\widehat\lam_1(T/4+t)}$, with $(1-\frac\alpha8)a\le a'\le(1+\frac\alpha8)a$ and $T\ge4a'$.

The schedule now has the form required by Lemma~\ref{lem:epca_schedule}; it remains to check its variance condition. Using $V=d\sigma^2$, $\sigma$ from \eqref{eq:dp_noise_variance} with $(R,\delta)\gets(1,\delta_*)$, and the sample bound \eqref{eq:private_epca_samples}, we obtain
\[
V\sum_{t=1}^T\eta_t^2
\le\frac{4Va^2}{\lam_1^2T}
=\frac{48dR_1R_2a^2\log\Par{\frac1{\delta_*}}}{n^2\eps^2\lam_1^2}
\le c,
\]
for a sufficiently small universal constant $c>0$. We may therefore apply Lemma~\ref{lem:epca_schedule} with $(\msig,n,a)\gets(\hmsig,T,a')$. For every fixed $S$ in $\calE_{\mathrm{cov}}$, it gives, with probability at least $\frac34$,
\[
1-\frac{\vw_T^\top\hmsig\vw_T}{\widehat\lam_1}
\le C_0\Par{\frac{\log(ed)}{a'}+\frac{Va'\log(ea')}{\widehat\lam_1^2T}}
\le\frac\alpha2.
\]
The error bound $\frac \alpha 2$ uses $R_1R_2=O(\nu^4d\log^2(\frac{nd}\alpha))$, $\log(\frac 1 {\delta_*})=O(\log(\frac{nd}{\alpha\delta}))$, and the sample bound \eqref{eq:private_epca_samples}, with $C_3$ and then $C$ sufficiently large. Let $\calE_{\mathrm{Oja}}$ denote this $\frac\alpha2$-ePCA guarantee for the coupled Oja output $\vw_T$; then $\Pr(\calE_{\mathrm{Oja}}^c\mid S)\le\frac14$.

We have obtained an $\frac\alpha2$-ePCA of $\hmsig$; it remains to transfer this guarantee to $\msig$. On $\calE_{\mathrm{cov}}\cap\calE_{\mathrm{cpl}}\cap\calE_{\mathrm{Oja}}$, coupling and the covariance bound give
\[
\vw_n^\top\msig\vw_n
\ge\Par{1-\frac\alpha2}\widehat\lam_1-\frac{\alpha\lam_1}8
\ge(1-\alpha)\lam_1.
\]
The private output $\vw_n$ is an $\alpha$-ePCA of $\msig$ on these events. The failure probability follows from
\[
\Pr(\calE_{\mathrm{cov}}^c)+\Pr(\calE_{\mathrm{cpl}}^c)
+\Pr(\calE_{\mathrm{Oja}}^c\cap\calE_{\mathrm{cov}})
\le\frac1{48}+\frac1{48}+\frac14<\frac13.
\]
\end{proof}

We remark that the success probability of Theorem~\ref{thm:private_epca} can be boosted using holdout samples, analogously to Corollary~\ref{cor:epca_whp}. For brevity, we omit this extension.

\begin{remark}[Gaussian specialization]\label{rem:private_epca_brown}
For Gaussian data, $\nu^2=\lam_1$, so Theorem~\ref{thm:private_epca} matches the sample complexity in Brown's Conjecture~1.1 \cite{Brown26} up to logarithmic factors when $\lam_1$ is known.
\end{remark}

\section{Experiments}
\label{sec:exp}

We conclude by providing empirical evaluations of Algorithm~\ref{alg:oja} (Section~\ref{ssec:oja_exp}) and Algorithm~\ref{alg:private_gapfree_batched_oja} (Section~\ref{ssec:privacy_exp}), to complement our theoretical results. Code for all experiments can be found \href{https://github.com/chutongyang98/Gap-Free-Streaming-PCA.git}{here}.

\subsection{Oja's algorithm}\label{ssec:oja_exp}
We first evaluate Oja's algorithm (Algorithm~\ref{alg:oja}) on synthetic streams with nearly tied leading eigenvalues, by comparing it against the top principal component of the empirical covariance. Note that this empirical estimator is not applicable in streaming settings, and serves only as a baseline.

Our experiments study the performance of these two algorithms under the same sample size. In our experiments, we set $d=50$ and set the population mean to $\msig = \mq\mlam\mq^\top$, where $\mlam$ is a diagonal matrix and $\mq$ is a Haar-distributed orthonormal matrix. The first three eigenvalues in $\mlam$ are fixed at $(1,0.99,0.98)$, and the remaining eigenvalues are independently drawn from $\Unif(0,0.95)$ and sorted in decreasing order. We generate our matrix stream as
\[\ma_t \defeq \msig+\frac{1}{\sqrt{d}}\mg_t,\quad [\mg_t]_{ij}\simiid\Nor(0,1).\]
It is straightforward to check that this is an instance of Model~\ref{model:oja_special} with $V=1$.

We set $\gamma=0.05$ and measure the cPCA success rate from Definition~\ref{def:cpca} with $\Delta=0.1$. For the output $\vw_n$ and an orthonormal eigenbasis $\Brack{\vv_i}_{i\in[d]}$, the cPCA error is 
\[\mathrm{err}(\vw_n,\msig)\defeq\sum_{i: \lambda_i<(1-\gamma)\lambda_1}\Par{\vv_i^\top\vw_n}^2.\]

We implemented Algorithm~\ref{alg:oja} with 
\[\eta_t = \frac{c}{\lambda_1(\beta+t)},\]
where $\lambda_1 =1$, and we performed a grid search for the pair of $c$ and $\beta$ that achieved the smallest mean cPCA error, over the choices \[c\in\{0.5,1,2,4,8,16,32\},\quad\beta\in\{0,1,3,10,30,100\}.\] 

The baseline returns a unit eigenvector corresponding to the largest eigenvalue of 
\[\widehat{\msig}\defeq\frac{1}{2n}\sum_{t\in[n]}\Par{\ma_t+\ma_t^\top},\]
i.e., the symmetrized empirical covariance. As shown in Figure~\ref{fig:experiment}, Oja's algorithm with tuned step sizes achieves results comparable to the baseline, but in the streaming setting.

\begin{figure}[t]
    \centering
    \begin{minipage}[b]{0.48\textwidth}
        \includegraphics[width=\textwidth]{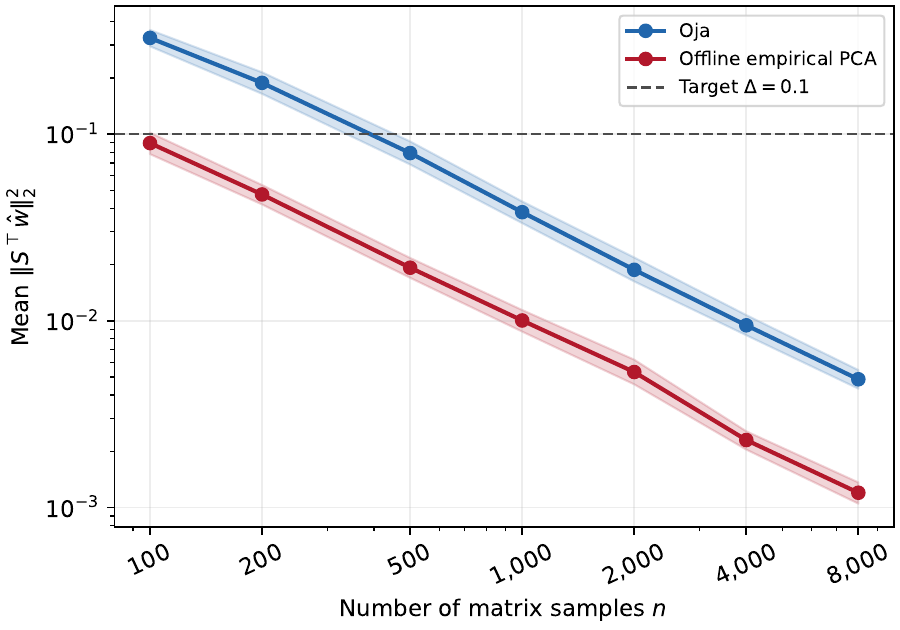}
    \end{minipage}
    \hfill
    \begin{minipage}[b]{0.48\textwidth}
        \includegraphics[width=\textwidth]{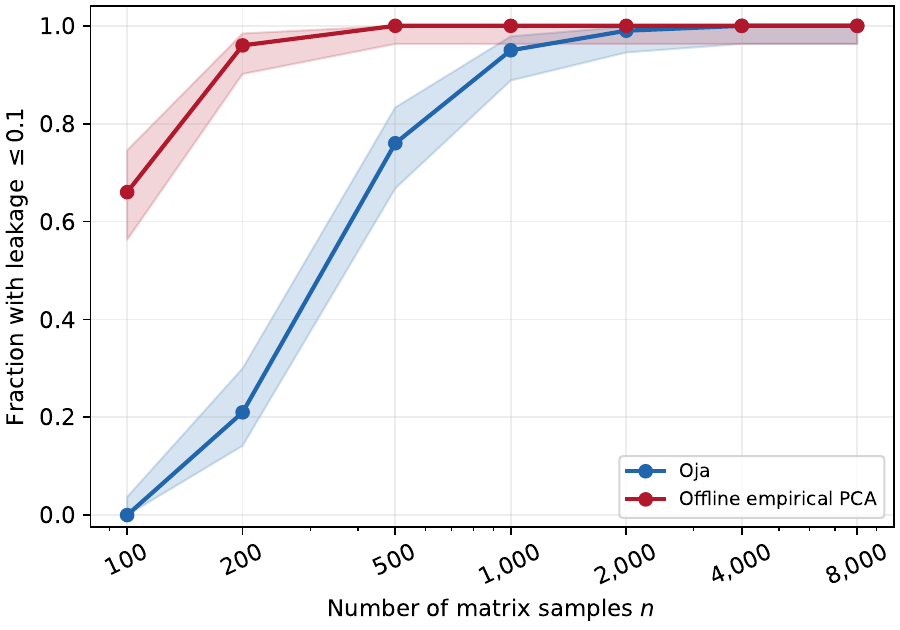}
    \end{minipage}
    \caption{Oja's algorithm with $\eta_t = \frac{16}{30+t}$ and offline empirical PCA across $100$ fresh paired trials. The left figure shows the cPCA error with $95\%$ Student-$t$ confidence intervals. The right figure shows the fraction of trials with cPCA error at most $\Delta = 0.1$ with $95\%$ Wilson intervals.}
    \label{fig:experiment}
\end{figure}

\subsection{Private PCA}\label{ssec:privacy_exp}

We next evaluate Algorithm~\ref{alg:private_gapfree_batched_oja} for private PCA on synthetic Gaussian samples, comparing it against the $\AG$ algorithm of \cite{DworkTTZ14}, which clips samples and noises the empirical covariance matrix entrywise. Up to logarithmic factors, our Theorem~\ref{thm:dp_pca_utility} and Theorem 6 of \cite{DworkTTZ14} show that the sample complexity of Algorithm~\ref{alg:private_gapfree_batched_oja} and $\AG$ respectively scale as\footnote{To see this bound for \cite{DworkTTZ14} in the Gaussian setting, after adapting their notation and scaling convention to ours, it suffices to plug in $n \gets d$ and $\sigma_1^2 - \sigma_2^2 \gets \frac{n\gamma}{d}$, and set the resulting $\sin^2$ error bound to $\Delta$.}
\begin{equation}\label{eq:dp_pca_rates}\approx \frac{d}{\gamma^2 \Delta} + \frac{d}{\eps\gamma\sqrt{\Delta}},\quad \approx \frac{d}{\gamma^2\Delta} + \frac{d^{1.5}}{\eps\gamma\sqrt{\Delta}}.\end{equation}
Observe that unless $d$ is somewhat large or $\gamma, \Delta$ are somewhat small, the identical first term in each of the above expressions dominates. Thus, we expect our algorithm to have improved performance over $\AG$ only in regimes with moderately large $d$ and small $\gamma, \Delta$.

In the following experiment, we set $d= 5000$ and vary $n\in\{10,20,30,50,100\}\times 10^6$. The population mean $\msig = \mq\mlam\mq^\top$ follows the exact same distribution as in Section~\ref{ssec:oja_exp}, i.e., 
the first three eigenvalues are $(1, 0.99,0.98)$, and the remaining eigenvalues are independently drawn from $\Unif(0,0.95)$. At each sample size, we run $20$ independent trials. We use the same cPCA error bound and success criterion of $\Delta = 0.1$ as before, and vary $\gamma \in \{\frac 1 4, \frac 1 {20}\}$ to measure the effect of this parameter. Finally, we set our DP parameters to $\eps = 1$ and $\delta = 10^{-6}$, and as hyperparameters to Algorithm~\ref{alg:private_gapfree_batched_oja}, we use 
\[R = 3,\quad T = 500, \quad \eta_t = \frac{32}{100+t},\quad \tau = \frac{1}{8}\sqrt{\frac{\Delta}{R}}.\]
The step sizes $\eta_t$ were picked using another grid search, selected from the same choices as used in Section~\ref{ssec:oja_exp}. We choose the clipping thresholds from~\eqref{eq:dp_pca_R1_R2} with $\nu = \lambda_1 = 1$ and $\zeta = 0.1$. The noise scale $\sigma$ is selected according to Lemma~\ref{lem:dp_pca_privacy}, which guarantees DP.

We next briefly describe the $\AG$ baseline from \cite{DworkTTZ14}. We used the same norm clipping threshold $R_1$, i.e., we follow Line~\ref{line:clip_1} of Algorithm~\ref{alg:private_gapfree_batched_oja} to produce clipped samples $\{\vs_i\}_{i \in [n]}$. $\AG$ then outputs a leading eigenvector of $\widehat{\msig}_{\text{clip}}+\mh$, where
\[\widehat{\msig}_{\text{clip}}\defeq\frac{1}{n}\sum_{i\in[n]}\vs_i\vs_i^\top,\quad \sigma_{\text{AG}}^2\defeq \frac{4R_1^2\log(1.25/\delta)}{n^2\eps^2},\]
and $\mh$ is a symmetric matrix with the upper triangle sampled i.i.d.\ from $\Nor(0,\sigma^2_{\text{AG}})$. 

In Figure~\ref{fig:experiment_private}, we show that Algorithm~\ref{alg:private_gapfree_batched_oja} achieves lower mean cPCA error than $\AG$ under the given parameters. As expected from \eqref{eq:dp_pca_rates}, Algorithm~\ref{alg:private_gapfree_batched_oja} performs better when  $\gamma$ is smaller. This improvement becomes less drastic when $n$ is very large, because rearranging \eqref{eq:dp_pca_rates} shows that
\[\Delta \approx \max\Par{\frac{d}{\gamma^2 n},\Par{\frac{d}{\eps\gamma n}}^2}\]
is dominated by the first term for large $n$. In such regimes, our error decay matches $\AG$.

We also note that, consistently with our theory, this finding appears to require a moderately large dimension to emerge: for example,
when $d = 3000$ and all other parameter settings remain fixed, $\AG$ achieves lower error than our algorithm.

\begin{figure}[t]
    \centering
    \begin{minipage}[b]{0.48\textwidth}
        \includegraphics[width=\textwidth]{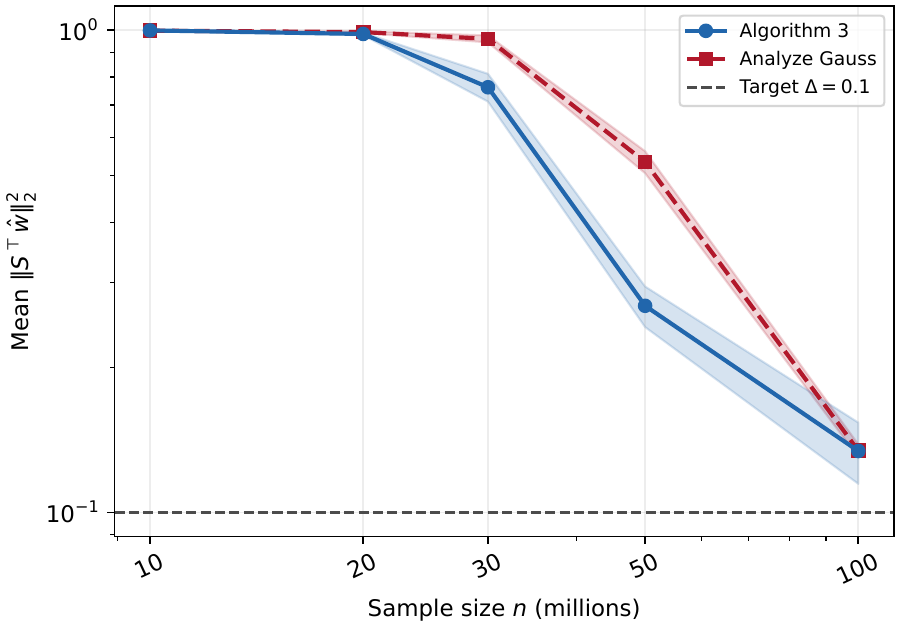}
    \end{minipage}
    \hfill
    \begin{minipage}[b]{0.48\textwidth}
        \includegraphics[width=\textwidth]{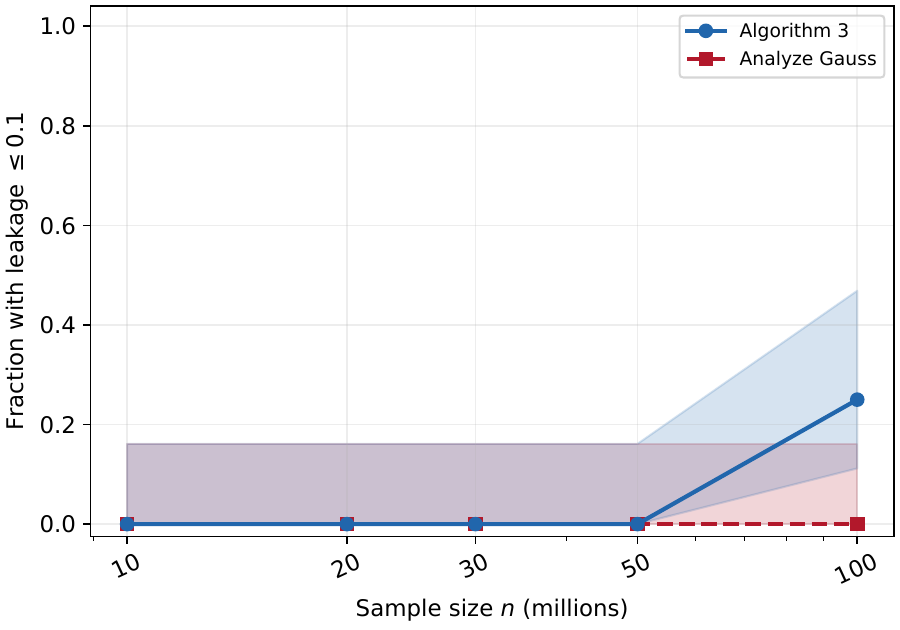}
    \end{minipage}
    \par\medskip

\begin{minipage}[b]{0.48\textwidth}
    \centering
    \includegraphics[width=\linewidth]{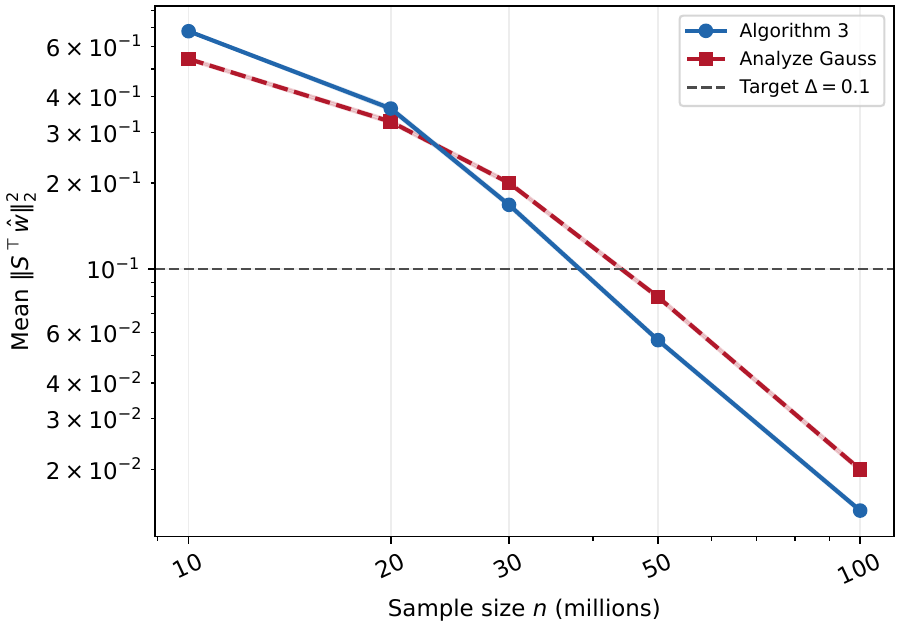}
\end{minipage}
\hfill
\begin{minipage}[b]{0.48\textwidth}
    \centering
    \includegraphics[width=\linewidth]{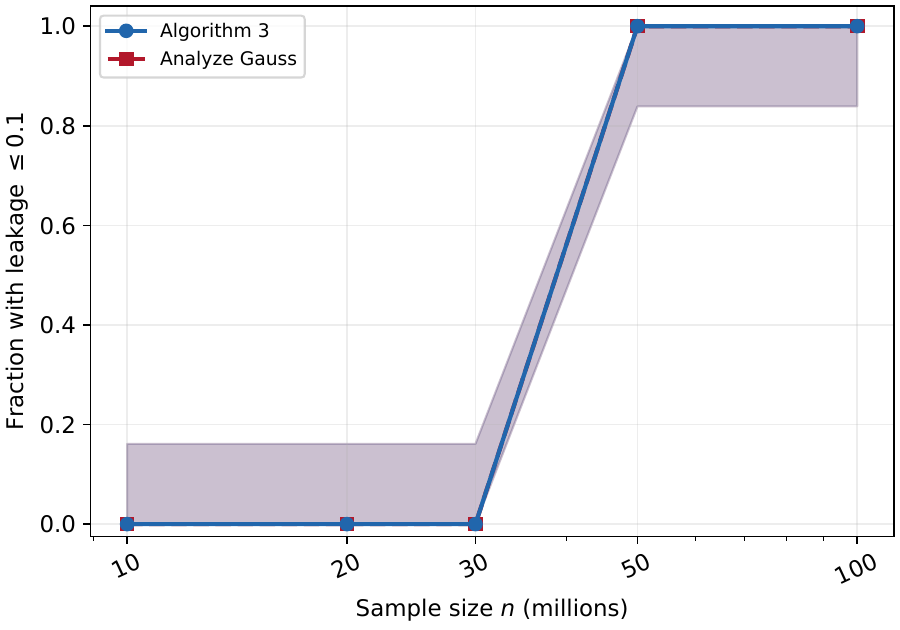}
\end{minipage}
    \caption{The left two figures show the cPCA error with $95\%$ Student-$t$ confidence intervals. The right two figures show the fraction of trials with cPCA error at most $\Delta = 0.1$ with $95\%$ Wilson intervals. The top figures are for $\gamma = 0.05$ and bottom figures are for $\gamma= 0.25$.}
    \label{fig:experiment_private}
    
\end{figure}

\section*{Acknowledgments}
SK and KT thank Ankit Pensia and Gavin Brown for several insightful discussions on this problem. SK and CY gratefully acknowledge support from the Amazon AI PhD Fellowship. We thank the NSF AI Institute for Foundations of Machine Learning (IFML) for supporting this project, and the Texas Advanced Computing Center (TACC) for providing the computing resources used.

\section*{AI Disclosure}
\phantomsection
\addcontentsline{toc}{section}{AI Disclosure}

The authors began the line of inquiry in this paper after discovering the connection between gap-free DP PCA and a gap-free Oja's algorithm in Section~\ref{sec:dp_pca}, and the lack of a gap-free, general rank analysis of Oja's algorithm. We used ChatGPT 5.5 and 5.6 Pro models to explore approaches for Theorem~\ref{thm:gapfree_oja}, primarily to aid with strategies for proving Lemma~\ref{lem:expected_total_energy_vary_step}, but the final proof strategy was developed by the authors. 
After completing all of our cPCA results, we learned about the statement of Conjecture 1.1 in \cite{Brown26} (which asked specifically for private ePCA) in personal communications with Gavin Brown. We then discovered the reduction in Proposition~\ref{prop:multiscale_cpca_epca} in conversations with ChatGPT 5.6 Pro, allowing us to extend our cPCA results to ePCA.
The manuscript was written solely by the authors, who take full responsibility for the organization and presentation of all results.

\newpage
\bibliographystyle{alpha}
\bibliography{refs}

\newpage
\appendix

\section{Deferred Proofs from Section~\ref{sec:epca}}
\label{sec:epca_proofs}

In this section, we prove Lemma~\ref{lem:epca_integrated_profile}, our multiscale cPCA guarantee on Algorithm~\ref{alg:oja}. 

Our strategy is to first bound $\E\normf{\mpp_u\mb_n}^2$ for each threshold $u$, and then give a
normalization argument needed to integrate this bound. Throughout, work under
Model~\ref{model:oja_special} with the products $\mb_s,\mc_s$ defined in
\eqref{eq:bi_def} and $0<\eta_s\le\frac1{\lam_1}$. For convenience, we also define
\[
q_n\defeq V\sum_{s\in[n]}\eta_s^2,
\qquad
Z_s\defeq\normf{\mc_s}^2,
\qquad
N_s(u)\defeq\E\normf{\mpp_u\mb_s}^2,
\]
where $\mpp_u$ is the projector from
Proposition~\ref{prop:multiscale_cpca_epca}.

\begin{lemma}
\label{lem:epca_expected_threshold}
For every $u\in(0,1)$,
\begin{align}
\frac{N_n(u)}{Z_n}
\le \exp(q_n)\Bigg(
&\min\Brace{
1,
d(1+q_n)
\exp\Par{
-\frac{u\lam_1}{2}\sum_{s\in[n]}\eta_s
}
}
+
V\sum_{i\in[n]}\eta_i^2
\exp\Par{
-\frac{u\lam_1}{2}\sum_{s=i+1}^n\eta_s
}
\Bigg).
\label{eq:epca_capped_expected_profile}
\end{align}
\end{lemma}

\begin{proof}
Our proof strategy is to bound the initial contribution and the variance introduced
at each update separately. For this, we first expand the second-moment recurrence,
then divide by $Z_n$ and bound the two contributions. Finally, we combine these
bounds with Lemma~\ref{lem:expected_total_energy_vary_step}.

Define $r_s(u)\defeq1+\eta_s(1-u)\lam_1$.
Taking the trace against $\mpp_u$ in the second moment recurrence
\eqref{eq:two_terms_vary} and using the variance bound in
Model~\ref{model:oja_special} then gives
\[
N_s(u)
\le
r_s(u)^2N_{s-1}(u)
+
V\eta_s^2\E\normf{\mb_{s-1}}^2.
\]
Iterating from $N_0(u)\le d$ and using
Lemma~\ref{lem:expected_total_energy_vary_step} with $t\gets i-1$ for $i>1$ (and $\mb_0=\mc_0=\id_d$ for $i=1$), we obtain
\begin{equation}
N_n(u)
\le
\underbrace{d\prod_{s\in[n]}r_s(u)^2}_{\text{initial contribution}}
+
\underbrace{V\exp(q_n)\sum_{i\in[n]}\eta_i^2 Z_{i-1}
\prod_{s=i+1}^n r_s(u)^2}_{\text{variance term}}.
\label{eq:epca_profile_recurrence}
\end{equation}
The initial contribution comes from $N_0(u)\le d$, multiplied by the recurrence
factors over all $n$ updates. The variance term sums the contribution introduced
at each update $i$, multiplied by the factors from the subsequent updates.
We first bound the initial contribution after dividing by $Z_n$, by comparing
with the product for $\lam_1$. The assumption $\eta_s\lam_1\le1$ gives
\begin{equation}
\Par{\frac{1+\eta_s\mu}{1+\eta_s\mu'}}^2
\le
\exp({-\eta_s(\mu'-\mu)})
\qquad
\text{for }0\le\mu\le\mu'\le\lam_1.
\label{eq:epca_scalar_contraction}
\end{equation}
Since $Z_n\ge\prod_{s\in[n]}(1+\eta_s\lam_1)^2$, the first term in \eqref{eq:epca_profile_recurrence}, divided by $Z_n$, is at
most
\[
d\exp\Par{
-u\lam_1\sum_{s\in[n]}\eta_s
}.
\]

It remains to bound the variance terms in \eqref{eq:epca_profile_recurrence}
after division by $Z_n$. For the $i$th summand, expand
\[
Z_{i-1}
=
\sum_{j\in[d]}
\prod_{s<i}(1+\eta_s\lam_j)^2
\]
and separate the eigenvalues at $(1-\frac u2)\lam_1$. Above this threshold,
we compare the remaining factors with those for the same eigenvalue in $Z_n$;
below it, we compare all factors with those for $\lam_1$. In both cases the
compared eigenvalues differ by at least $\frac{u\lam_1} 2$. For
$\lam_j\ge(1-\frac u2)\lam_1$, 
\eqref{eq:epca_scalar_contraction} gives
\[
\prod_{s<i}(1+\eta_s\lam_j)^2
\prod_{s=i+1}^n r_s(u)^2
\le
\exp\Par{
-\frac{u\lam_1}{2}\sum_{s=i+1}^n\eta_s
}
\prod_{s\in[n]}(1+\eta_s\lam_j)^2.
\]
Summing over $j$ with $\lam_j\ge(1-\frac u2)\lam_1$ bounds their total
contribution by $Z_n\exp(-\frac{u\lam_1}{2}\sum_{s=i+1}^n\eta_s)$.

For $\lam_j<(1-\frac u2)\lam_1$, the scalar ratio bound in
\eqref{eq:epca_scalar_contraction} and $\prod_{s\in[n]}(1+\eta_s\lam_1)^2\le Z_n$ give
\[
\begin{aligned}
\prod_{s<i}(1+\eta_s\lam_j)^2\prod_{s=i+1}^n r_s(u)^2
&\le \prod_{s\in[n]}\Par{1+\eta_s\Par{1-\frac u2}\lam_1}^2\le Z_n\exp\Par{-\frac{u\lam_1}{2}\sum_{s\in[n]}\eta_s}.
\end{aligned}
\]
There are at most $d$ such $j$, so the two ranges together give
\[
\frac{
Z_{i-1}\prod_{s=i+1}^n r_s(u)^2
}{
Z_n
}
\le
\exp\Par{
-\frac{u\lam_1}{2}\sum_{s=i+1}^n\eta_s
}
+
d\exp\Par{
-\frac{u\lam_1}{2}\sum_{s\in[n]}\eta_s
}.
\]

Substituting the bounds for both terms into \eqref{eq:epca_profile_recurrence}
and using $V\sum_{i\in[n]}\eta_i^2=q_n$ gives the claimed bound without
the minimum with $1$. To obtain that minimum, apply Lemma~\ref{lem:expected_total_energy_vary_step} with $t\gets n$ and use $\normf{\mpp_u\mb_n}\le\normf{\mb_n}$ to get $\frac{N_n(u)}{Z_n}\le\exp(q_n)$.
Combining this bound on $\frac{N_n(u)}{Z_n}$ with the bound obtained from \eqref{eq:epca_profile_recurrence}, using
$\min\{1,x+y\}\le\min\{1,x\}+y$ for $x,y\ge0$ proves the claim.
\end{proof}

\begin{lemma}
\label{lem:epca_profile_normalization}
There are universal constants $c,C>0$ such that, if $q_n\le c$, then for any
fixed $\md\succeq\0_{d\times d}$ and independent
$\vg\sim\Nor(\0_d,\id_d)$, with probability at least $\frac34$,
$\mb_n\vg\ne\0_d$ and
\[
\frac{
\vg^\top\mb_n^\top\md\mb_n\vg
}{
\norm{\mb_n\vg}_2^2
}
\le
C
\frac{
\E\Tr(\md\mb_n\mb_n^\top)
}{
Z_n
}.
\]
\end{lemma}

\begin{proof}
Our proof strategy is to control the numerator and denominator of the normalized
output using bounds on $\mb_n$. We first lower bound $\normf{\mb_n}^2$ and
upper bound $\Tr(\md\mb_n\mb_n^\top)$. We then condition on $\mb_n$, apply
the Gaussian quadratic-form bounds, and combine the three events.

Choose $c$ so that $\exp(c)-1\le\frac1{48}$. Since $\E\mb_n=\mc_n$, Lemma~\ref{lem:expected_total_energy_vary_step} with $t\gets n$ gives
\[
\E\normf{\mb_n-\mc_n}^2
=\E\normf{\mb_n}^2-Z_n
\le(\exp(q_n)-1)Z_n.
\]
Markov's inequality gives events
\[
\mathcal E_B
\defeq
\Brace{
\normf{\mb_n-\mc_n}^2\le \frac{Z_n}4
},
\qquad
\mathcal E_D
\defeq
\Brace{
\Tr(\md\mb_n\mb_n^\top)
\le
12\E\Tr(\md\mb_n\mb_n^\top)
},
\]
each with failure probability at most $\frac1{12}$. Fix $\mb_n$ in
$\mathcal E_B\cap\mathcal E_D$; then $\normf{\mb_n}^2\ge \frac{Z_n} 4$, and $\vg$ remains an independent standard Gaussian.

We have thus obtained the required bounds on $\mb_n$; it remains to control the ratio of quadratic forms in $\vg$.
The proof of Lemma~\ref{lem:gapfree_one_step_power} applies with $\mh=\mb_n^\top\mb_n$ and $\mk=\mb_n^\top\md\mb_n$.
Its Gaussian quadratic-form bounds require only
$\mh,\mk\succeq\0_{d\times d}$ and $\mh\ne\0_{d\times d}$. With
$\zeta=\frac1{12}$, it gives an event
\[
\mathcal E_g
\defeq
\Brace{
\norm{\mb_n\vg}_2^2>0,
\quad
\frac{
\vg^\top\mb_n^\top\md\mb_n\vg
}{
\norm{\mb_n\vg}_2^2
}
\le
C_0
\frac{
\Tr(\md\mb_n\mb_n^\top)
}{
\normf{\mb_n}^2
}
},
\]
with conditional failure probability at most $\frac1{12}$, for a universal
constant $C_0$. On $\mathcal E_B\cap\mathcal E_D\cap\mathcal E_g$, the
claimed bound holds with $C=48C_0$. The total failure probability is at most
$3\cdot\frac1{12}=\frac14$.
\end{proof}

\begin{proof}[Proof of Lemma~\ref{lem:epca_integrated_profile}]
Our strategy is to bound the integral by an expected trace using
Lemma~\ref{lem:epca_profile_normalization}, then estimate this trace using
Lemma~\ref{lem:epca_expected_threshold}. For this, we first write the integral
as a quadratic form.

Set $\md\defeq\id_d-\frac{\msig}{\lam_1}\succeq\0_{d\times d}$.
Proposition~\ref{prop:multiscale_cpca_epca}, applied to $\msig$, gives
$\md=\int_0^1\mpp_u\,\dd u$.
On the event $\mb_n\vg\ne\0_d$, write
$\vw_n=\mb_n\vg/\norm{\mb_n\vg}_2$ and
$\Delta_u=\norm{\mpp_u\vw_n}_2^2$. Then
\[
\int_0^1\Delta_u\,\dd u
=
\vw_n^\top\md\vw_n
=
\frac{
\vg^\top\mb_n^\top\md\mb_n\vg
}{
\norm{\mb_n\vg}_2^2
}.
\]
Applying Lemma~\ref{lem:epca_profile_normalization} with
$\md\gets\id_d-\frac{\msig}{\lam_1}$ and using linearity of trace and expectation
therefore gives, with probability at least $\frac34$,
\begin{equation}\label{eq:epca_normalized_integral}
\int_0^1\Delta_u\,\dd u
\le C\frac{\E\Tr(\md\mb_n\mb_n^\top)}{Z_n}
=C\int_0^1\frac{N_n(u)}{Z_n}\,\dd u.
\end{equation}

We have reduced the desired bound to the integral of $\frac{N_n(u)}{Z_n}$ in
\eqref{eq:epca_normalized_integral}. It remains to apply
Lemma~\ref{lem:epca_expected_threshold} and integrate its two terms.
For $A\ge1$ and $x>0$, direct integration gives
\begin{equation}\label{eq:epca_elementary_integrals}
\int_0^1\min\Brace{1,A\exp\Par{-\frac{ux}2}}\,\dd u
\le\frac{2(1+\log A)}{x},
\qquad
\int_0^1\exp\Par{-\frac{ux}2}\,\dd u
\le2\min\Brace{1,\frac1x}.
\end{equation}
To conclude, we apply Lemma~\ref{lem:epca_expected_threshold} for each $u\in(0,1)$.
For its first term, use the first integration bound in
\eqref{eq:epca_elementary_integrals} with
$A=d(1+q_n)$ and $x=\lam_1\sum_{s\in[n]}\eta_s$.
For its $i^{\text{th}}$ variance summand with $i<n$, use the second integration bound
with $x=\lam_1\sum_{s=i+1}^n\eta_s$; for $i=n$, the exponential is
identically $1$. Taking $c\le1$, we have $\exp(q_n)=O(1)$ and
$1+\log(d(1+q_n))=O(\log(ed))$, so
\[
\int_0^1\frac{N_n(u)}{Z_n}\,\dd u
\le C'\Par{
\frac{\log(ed)}{\lam_1\sum_{s\in[n]}\eta_s}
+V\sum_{i\in[n]}\eta_i^2
\min\Brace{1,\frac1{\lam_1\sum_{s=i+1}^n\eta_s}}
},
\]
where $C'>0$ is universal and the minimum for $i=n$ is interpreted as $1$.
Substituting this bound on $\int_0^1\frac{N_n(u)}{Z_n}\,\dd u$ into
\eqref{eq:epca_normalized_integral} proves \eqref{eq:epca_integrated_profile}.
\end{proof}

\end{document}